\documentclass{article}

\usepackage[top=1in, bottom=1in, left=1in, right=1in]{geometry}

\usepackage{comment,url,algorithmic,graphicx,subcaption,relsize}
\usepackage{amssymb,amsfonts,amsmath,amsthm,amscd,dsfont,mathrsfs,mathtools,nicefrac}
\usepackage{float,psfrag,epsfig,color,xcolor,url,hyperref}
\usepackage{epstopdf,bbm,enumitem}
\usepackage[toc,page]{appendix}
\newcommand{\iid}{\stackrel{\mathrm{\tiny{i.i.d.}}}{\sim}}

\def\ddefloop#1{\ifx\ddefloop#1\else\ddef{#1}\expandafter\ddefloop\fi}

\def\ddef#1{\expandafter\def\csname bb#1\endcsname{\ensuremath{\mathbb{#1}}}}
\ddefloop ABCDEFGHIJKLMNOPQRSTUVWXYZ\ddefloop

\def\ddef#1{\expandafter\def\csname c#1\endcsname{\ensuremath{\mathcal{#1}}}}
\ddefloop ABCDEFGHIJKLMNOPQRSTUVWXYZ\ddefloop

\def\ddef#1{\expandafter\def\csname s#1\endcsname{\ensuremath{\mathscr{#1}}}}
\ddefloop ABCDEFGHIJKLMNOPQRSTUVWXYZ\ddefloop

\def\ddef#1{\expandafter\def\csname v#1\endcsname{\ensuremath{\boldsymbol{#1}}}}
\ddefloop ABCDEFGHIJKLMNOPQRSTUVWXYZabcdefghijklmnopqrstuvwxyz\ddefloop

\def\ddef#1{\expandafter\def\csname v#1\endcsname{\ensuremath{\boldsymbol{\csname #1\endcsname}}}}
\ddefloop {alpha}{beta}{gamma}{delta}{epsilon}{varepsilon}{zeta}{eta}{theta}{vartheta}{iota}{kappa}{lambda}{mu}{nu}{xi}{pi}{varpi}{rho}{varrho}{sigma}{varsigma}{tau}{upsilon}{phi}{varphi}{chi}{psi}{omega}{Gamma}{Delta}{Theta}{Lambda}{Xi}{Pi}{Sigma}{varSigma}{Upsilon}{Phi}{Psi}{Omega}{ell}\ddefloop

\def\balign#1\ealign{\begin{align}#1\end{align}}
\def\baligns#1\ealigns{\begin{align*}#1\end{align*}}
\def\balignat#1\ealign{\begin{alignat}#1\end{alignat}}
\def\balignats#1\ealigns{\begin{alignat*}#1\end{alignat*}}
\def\bitemize#1\eitemize{\begin{itemize}#1\end{itemize}}
\def\benumerate#1\eenumerate{\begin{enumerate}#1\end{enumerate}}

\newenvironment{talign*}
 {\csname align*\endcsname}
 {\endalign}
\newenvironment{talign}
 {\csname align\endcsname}
 {\endalign}

\def\balignst#1\ealignst{\begin{talign*}#1\end{talign*}}
\def\balignt#1\ealignt{\begin{talign}#1\end{talign}}
\let\originalleft\left
\let\originalright\right
\renewcommand{\left}{\mathopen{}\mathclose\bgroup\originalleft}
\renewcommand{\right}{\aftergroup\egroup\originalright}

\def\tinycitep*#1{{\tiny\citep*{#1}}}
\def\tinycitealt*#1{{\tiny\citealt*{#1}}}
\def\tinycite*#1{{\tiny\cite*{#1}}}
\def\smallcitep*#1{{\scriptsize\citep*{#1}}}
\def\smallcitealt*#1{{\scriptsize\citealt*{#1}}}
\def\smallcite*#1{{\scriptsize\cite*{#1}}}

\def\mbb#1{\mathbb{#1}}

\newcommand{\norm}[1]{\left\lVert#1\right\rVert}

\theoremstyle{plain}

\def\R{\mathbb{R}}
\def\N{\mathbb{N}}
\def\<{\left\langle} % Angle brackets
\def\>{\right\rangle}

\def\norm#1{\left\|{#1}\right\|} % A norm with 1 argument
\def\polylog{\operatorname{polylog}}
\def\E{\mbb{E}} % Expectation symbol

\def\P{\mbb{P}} % Probability symbol

\newcommand{\Exp}[1]{\operatorname{exp}\left({#1}\right)} % Exponential 
\DeclareSymbolFont{rsfs}{U}{rsfs}{m}{n}
\DeclareSymbolFontAlphabet{\mathscrsfs}{rsfs}
\newcommand{\He}{\mathsf{He}}

\ifdefined\nonewproofenvironments\else
\ifdefined\ispres\else
\newtheorem{theo}{Theorem}

\newtheorem{defi}[theo]{Definition}

\newtheorem{prop}[theo]{Proposition}

\newtheorem{theorem}{Theorem}
\newtheorem{lemma}[theo]{Lemma}

\newtheorem{proposition}[theo]{Proposition}

\newtheorem*{theorem*}{\textbf{Theorem}}

\renewenvironment{proof}{\noindent\textbf{Proof.}\hspace*{.3em}}{\qed\\}
\newenvironment{proof-sketch}{\noindent\textbf{Proof Sketch}
  \hspace*{1em}}{\qed\bigskip\\}
\newenvironment{proof-idea}{\noindent\textbf{Proof Idea}
  \hspace*{1em}}{\qed\bigskip\\}
\newenvironment{proof-of-lemma}[1][{}]{\noindent\textbf{Proof of Lemma {#1}}
  \hspace*{1em}}{\qed\\}
\newenvironment{proof-of-theorem}[1][{}]{\noindent\textbf{Proof of Theorem {#1}}
  \hspace*{1em}}{\qed\\}
\newenvironment{proof-attempt}{\noindent\textbf{Proof Attempt}
  \hspace*{1em}}{\qed\bigskip\\}
\fi

\newtheorem{assumption}{Assumption}

\newenvironment{proofof}[1][{}]{\par\noindent{\bf Proof of {#1}. }  }{\qed\bigskip}   
   
\fi
\makeatletter
\@addtoreset{equation}{section}
\makeatother

\hypersetup{
  colorlinks,
  linkcolor={blue!50!black},
  citecolor={blue!50!black},
  urlcolor={blue!80!black}
}

\mathtoolsset{showonlyrefs}

\usepackage{caption}
\usepackage{authblk}
\usepackage{etoc}
\usepackage{wrapfig}
\usepackage[linesnumbered,ruled,vlined]{algorithm2e}
\usepackage{tikz}

\renewcommand{\contentsname}{Table of Contents}

\newcommand{\citep}[1]{\cite{#1}}
\newcommand{\citet}[1]{\cite{#1}}

\makeatletter
\renewcommand{\paragraph}{%
  \@startsection{paragraph}{4}%
  {\z@}{1.6ex \@plus 1ex \@minus .2ex}{-1em}%
  {\normalfont\normalsize\bfseries}%
}
\makeatother

\title{
Neural network learns low-dimensional polynomials with SGD \\ near the information-theoretic limit}

\author{
Jason D.~Lee\thanks{Princeton University. \texttt{jasonlee@princeton.edu}.} ,\,\,
Kazusato Oko\thanks{University of California, Berkeley and RIKEN AIP. \texttt{oko@berkeley.edu}.},\,\,
Taiji Suzuki\thanks{University of Tokyo and RIKEN AIP. \texttt{taiji@mist.i.u-tokyo.ac.jp}.} ,\,\,
Denny Wu\thanks{New York University and Flatiron Institute. \texttt{dennywu@nyu.edu}.}
\vspace{-2.5mm}
}

\ifdefined\usebigfont

\usepackage{times}
\usepackage[fontsize=13pt]{scrextend}
\makeatletter
\@ifpackageloaded{geometry}{\AtBeginDocument{\newgeometry{letterpaper,left=1.56in,right=1.56in,top=1.71in,bottom=1.77in}}}{\usepackage[letterpaper,left=1.56in,right=1.56in,top=1.71in,bottom=1.77in]{geometry}}
\AtBeginDocument{\newgeometry{letterpaper,left=1.56in,right=1.56in,top=1.71in,bottom=1.77in}}
\else
\fi

\begin{document}
\etocdepthtag.toc{mtchapter}
\etocsettagdepth{mtchapter}{subsection}
\etocsettagdepth{mtappendix}{none}

\maketitle 

\vspace{-2mm}

\begin{abstract}

We study the problem of gradient descent learning of a single-index target function $f_*(\boldsymbol{x}) = \textstyle\sigma_*\left(\langle\boldsymbol{x},\boldsymbol{\theta}\rangle\right)$ under isotropic Gaussian data in $\mathbb{R}^d$, 
where the unknown link function $\sigma_*:\mathbb{R}\to\mathbb{R}$ has information exponent $p$ (defined as the lowest degree in the Hermite expansion). Prior works showed that gradient-based training of neural networks can learn this target with $n\gtrsim d^{\Theta(p)}$ samples, and such complexity is predicted to be necessary by the correlational statistical query lower bound. 
Surprisingly, we prove that a two-layer neural network optimized by an SGD-based algorithm (on the squared loss) learns $f_*$ with a complexity that is not governed by the information exponent. Specifically, for arbitrary polynomial single-index models, we establish a sample and runtime complexity of $n \simeq T = \Theta(d\!\cdot\! \mathrm{polylog} d)$, where $\Theta(\cdot)$ hides a constant only depending on the degree of $\sigma_*$; this dimension dependence matches the information theoretic limit up to polylogarithmic factors. More generally, we show that $n\gtrsim d^{(p_*-1)\vee 1}$ samples are sufficient to achieve low generalization error, where $p_* \le p$ is the \textit{generative exponent} of the link function. Core to our analysis is the reuse of minibatch in the gradient computation, which gives rise to higher-order information beyond correlational queries. 

\end{abstract}

\allowdisplaybreaks

\section{Introduction}
\label{sec:intro}

Single-index models are a classical class of functions that capture low-dimensional structure in the learning problem. To efficiently estimate such functions, the learning algorithm should extract the relevant (one-dimensional) subspace from high-dimensional observations; hence this problem setting has been extensively studied in deep learning theory \cite{bai2019beyond,ba2022high,bietti2022learning,mousavi2023neural,mahankali2024beyond,wang2024nonlinear}, to examine the adaptivity to low-dimensional targets and benefit of representation learning in neural networks (NNs) optimized by gradient descent (GD). 
In this work we study the learning of a single-index target function under isotropic Gaussian data:
\begin{align}
    y_i = f_*(\vx_i) + 	\varsigma_i, \quad 
    f_*(\vx_i) = \sigma_*(\langle\vx_i,\vtheta\rangle), 
    \quad
    \vx_i\iid\cN(0,\vI_d), 
    \label{eq:teacher}
\end{align}
where $\varsigma_i$ is i.i.d.~label noise, $\vtheta\in\R^d$ is the direction of index features, and we assume the link function $\sigma_*:\R\to\R$ has information exponent $p\in\N_+$ defined as the index of the first non-zero coefficient in the Hermite expansion (see Definition~\ref{def:information-exponent}). 
 
Equation \eqref{eq:teacher} requires the estimation of the one-dimensional link function $\sigma_*$ and the relevant direction $\vtheta$; it is known that learning is information theoretically possible with $n\gtrsim d$ training examples \cite{dudeja2024statistical,damian2024computational}. 
Indeed, when $\sigma_*$ is polynomial, such statistical complexity can be achieved up to logarithmic factors by a tailored algorithm that exploit the structure of low-dimensional target \cite{chen2020learning}. 
On the other hand, for gradient-based training of two-layer NNs, existing works established a sample complexity of $n\gtrsim d^{\Theta(p)}$
\cite{arous2021online,bietti2022learning,damian2023smoothing}, which presents a gap between the information theoretic limit and what is computationally achievable by (S)GD. Such a gap is also predicted by the correlational statistical query (CSQ) lower bound \cite{damian2022neural,abbe2023sgd}, which roughly states that for a CSQ algorithm to learn (isotropic) Gaussian single-index models using less than exponential compute, 
a sample size of $n\gtrsim d^{p/2}$ is necessary. 

Although CSQ lower bounds are frequently cited to imply a fundamental barrier of learning via SGD (with the squared/correlation loss), strictly speaking, the CSQ model does not include empirical risk minimization with gradient descent, due to the non-adversarial noise and existence of non-correlational terms in the gradient computation. 
Very recently, \cite{dandi2024benefits} exploited higher-order terms in the gradient update arising from the reuse of the same training data, and showed that for certain link functions with high information exponent ($p>2$), two-layer NNs may still achieve weak recovery (i.e., nontrivial overlap with $\vtheta$) after two GD steps with $\Theta(d)$ batch size. 
While this presents evidence that GD-trained NNs can learn $f_*$ beyond the sample complexity suggested by the CSQ lower bound, the weak recovery statement in \cite{dandi2024benefits} may not translate to statistical guarantees; moreover, the class of functions where SGD can achieve vanishing generalization error is not fully characterized, as only a few specific examples of link functions are discussed.

\begin{figure}[t]
% \vspace{-3.5mm}
\centering
\begin{minipage}[t]{0.48\linewidth}
\centering
{\includegraphics[height=0.72\textwidth]{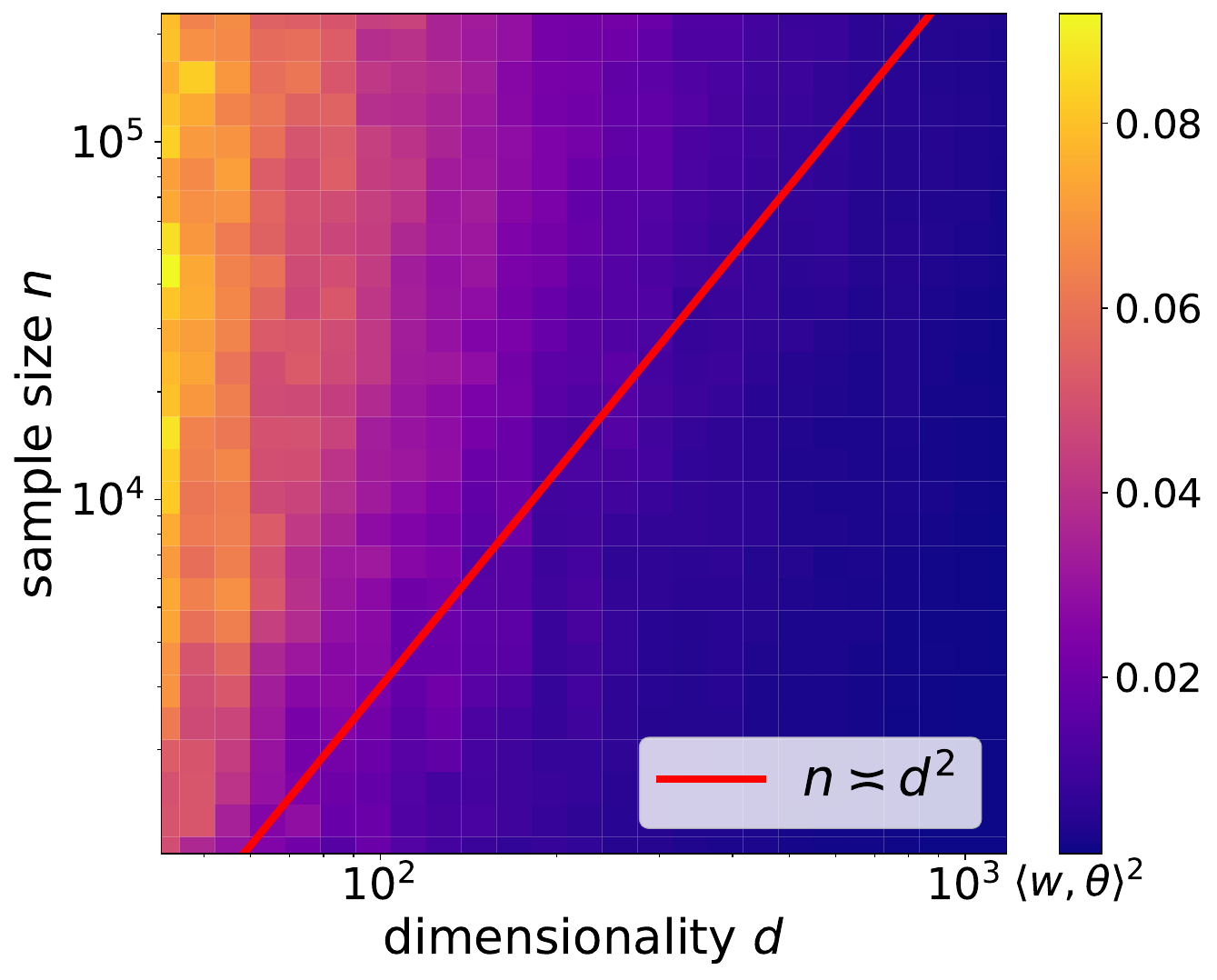}}  \\
\small (a) Online SGD (weak recovery). 
\end{minipage}%
\begin{minipage}[t]{0.48\linewidth}
\centering 
{\includegraphics[height=0.72\textwidth]{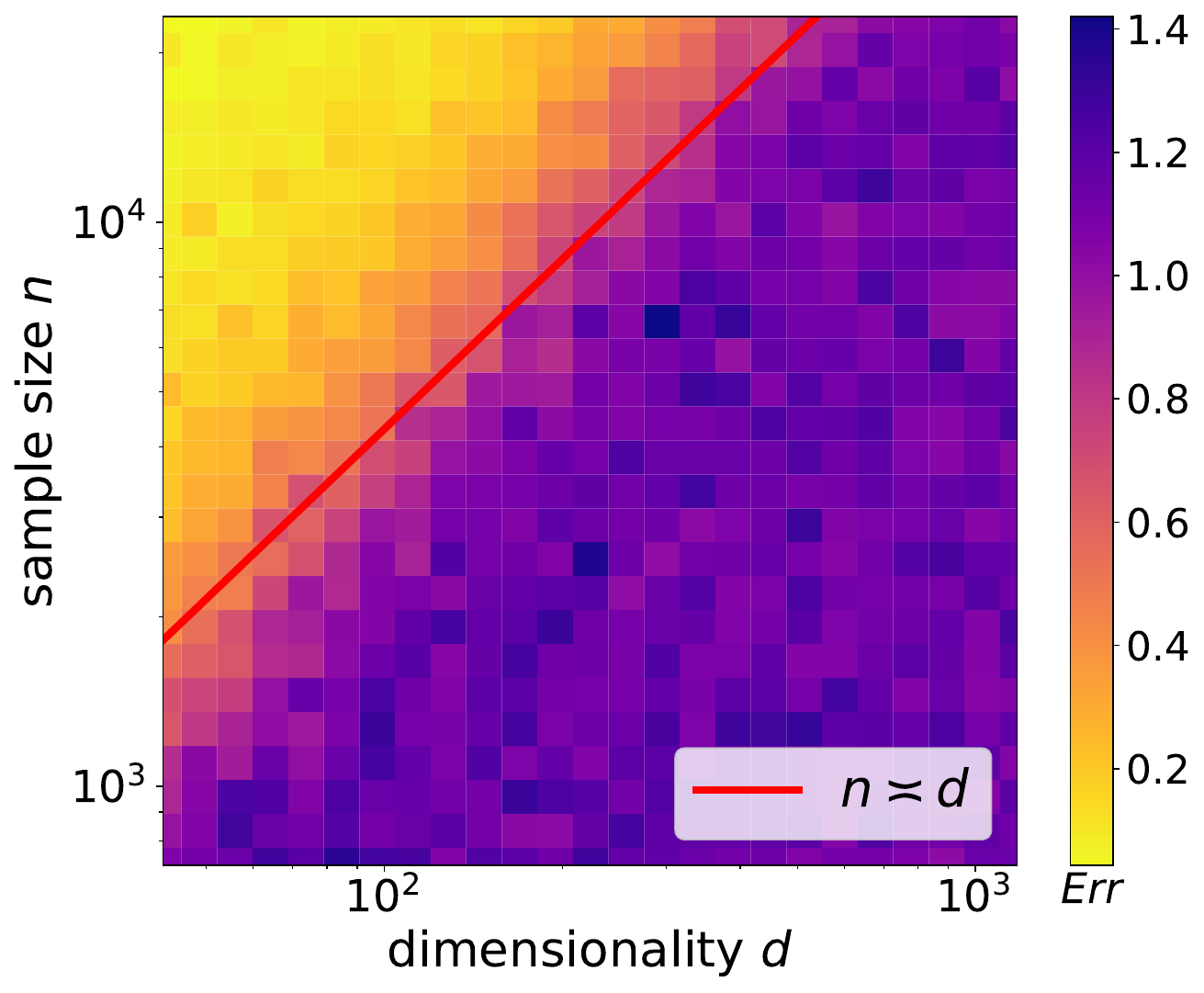}} \\ 
\small (b) Same-batch GD (generalization error). 
\end{minipage}%
\caption{\small We train a ReLU NN \eqref{eq:student} with $N=1024$ neurons using SGD (squared loss) with step size $\eta=1/d$ to learn a single-index target $f_*(\vx) = \He_3(\langle\vx,\vtheta\rangle)$; heatmaps are values averaged over 10 runs. $(a)$ online SGD with batch size $B=8$; $(b)$ GD on the same batch of size $n$ for $T=2^{14}$ steps. For online SGD we only report weak recovery (i.e., averaged overlap between neuron $\vw$ and target $\vtheta$) since the test error does not drop.
}   
% \vspace{-1mm}
\label{fig:intro}  
\end{figure}  

Given the existence of (non-NN) algorithms that learn any single-index polynomials in $n=\tilde{O}(d)$ samples \cite{chen2020learning} regardless of the information exponent $p$, and more generally, non-CSQ algorithms with a sample complexity surpassing the CSQ lower bound \cite{damian2024computational}, it is natural to ask if gradient-based training of NNs can achieve similar statistical efficiency for this function class. Motivated by observations in \cite{dandi2024benefits} that SGD with reused data may break the ``curse of information exponent'', we aim to address the question: 
\begin{center}
    \textit{Can NN optimized by SGD with reused batch learn single-index $f_*$ beyond the CSQ lower bound? 
    And for polynomial $\sigma_*$, can learning succeed near the information-theoretic limit $n\simeq d$?} 
\end{center}

Empirically, the separation between one-pass (online) and multi-pass SGD is clearly observed in Figure~\ref{fig:intro}, where we trained the same two-layer ReLU neural network to learn a single-index polynomial with information exponent $p=3$. We see that SGD with reused data (Figure~\ref{fig:intro}(b)) reaches low test error using roughly $n\simeq d$ samples, whereas online SGD fails to achieve even weak recovery with much larger sample size $n=\Omega(d^2)$. 
Our main contribution is to establish this improved statistical complexity for two-layer NNs trained by a variant of SGD with reused training data.

\subsection{Our Contributions}

We answer the above question in the affirmative by showing
that SGD training (with the squared loss) on a natural class of shallow NNs can achieve small generalization error using polynomial compute and a sample complexity that is not governed by the information exponent, if we employ a layer-wise optimization procedure (analogous to that in \cite{ba2022high,damian2022neural,abbe2023sgd}) and reuse of the same minibatch. The core insight is that SGD can implement a full statistical query (SQ) algorithm that goes beyond CSQ, despite the correlational structure of the squared loss. Our main finding is summarized by the following theorem. 

\begin{theorem*}[informal] 
A shallow NN with $N=\tilde{\Theta}_d(1)$ neurons can learn arbitrary single-index models up to small population loss: $\E_{\vx}[(f_{\vTheta}(\vx) - f_*(\vx))^2] = o_{d,\P}(1)$, if we employ an SGD-based algorithm (with reused training data) to minimize the squared loss objective, with a sample and runtime complexity of $n,T = \tilde{\Theta}_d(d^{(p_*-1)\vee 1})$, where $p_*$ is the \textit{generative exponent} of the link function $\sigma_*$. 
\end{theorem*}

Note that the generative exponent \cite{damian2024computational} is defined as the \textit{minimum} information exponent of the link function $\sigma_*$ after arbitrary $L^2$ transformation, and hence by definition $p_*\le p$  (equality is achieved by the identity transformation). 
We make the following remarks on our main result. 
\begin{itemize}[leftmargin=*,topsep=1mm,itemsep=1mm]
\item We know that $p_*\le 2$ for arbitrary \textit{polynomial} link functions. Therefore, the theorem suggests that NN + SGD with reused batch can learn single-index polynomials with a sample complexity $n=\tilde{O}_d(d)$ which is information theoretically optimal up to polylogarithmic factors, hence matching the efficiency of SQ algorithms tailored for low-dimensional polynomial regression \cite{chen2020learning}.
\item For non-polynomial $\sigma_*$ with high generative exponent $p_*>2$, our sample complexity $n\gtrsim d^{p_*-1}$ can be interpreted as an SQ version of the online SGD result in \cite{arous2021online}. Since the information exponent $p$ can be arbitrarily larger than the generative exponent $p_*$, our main theorem disproves a conjecture in \cite{abbe2023sgd} stating that $n\asymp d^{p/2}$ is the optimal sample complexity for empirical risk minimization with SGD on the squared loss / correlation loss.  
\item A key observation in our analysis is that with suitable activation function, SGD with reused batch can go beyond correlational queries and implement (a subclass of) SQ algorithms. This enables polynomial transformations to the labels that reduce the information exponent, and therefore optimization can escape the high-entropy ``equator'' at initialization in polylogarithmic time. 
\end{itemize}
Upon completion of this work, we became aware of the preprint \cite{arnaboldi2024repetita}
showing weak recovery (for polynomial targets with $p_*\le 2$) with similar sample complexity, also by exploiting the reuse of training data. Our work was conducted independently and simultaneously.

\begin{figure}[t]
% \vspace{-4mm}
\centering
\begin{minipage}{\textwidth}
\begin{tikzpicture}[x=\textwidth/18]
\usetikzlibrary{arrows.meta}
  % Define the style for the timeline and markers
  \tikzset{
    timeline/.style={very thick, line width=1.8pt, -{Triangle[width=7.5pt,length=10pt]}},
    major tick/.style={thick, line width=1.5pt},
    minor tick/.style={thin, line width=1pt},
    tick label/.style={above, outer sep=2mm, scale=0.6}, % Adjusted the scale for text to fit better
    node style/.style={align=center}
  }

  % Draw the timeline
  \draw[timeline] (0,0) -- (15.92,0);
  % Add minor ticks and labels
  \draw[minor tick] (0, -0.15) -- (0, 0.15);
  % \node[tick label] at (1, 0.1) {1};
  \draw[minor tick] (1, -0.15) -- (1, 0.15);
  \draw[minor tick] (6.8, -0.15) -- (6.8, 0.15);
  \draw[minor tick] (11.2, -0.15) -- (11.2, 0.15);
  \draw[minor tick] (14.6, -0.15) -- (14.6, 0.15);

  % Labels below timeline
  \node[node style,font=\scriptsize] at (-1, 0) {Information \\ theoretic limit};
  \node[node style,font=\small] at (1, 0.7) {
  SGD + batch reuse [{\color{red!50!black}\textbf{This work}}] \\
  SQ algorithm \cite{chen2020learning}};
  \node[node style,font=\small] at (6.8, 0.7) {Smoothed SGD \cite{damian2023smoothing} \\ CSQ lower bound \cite{damian2022neural}};
  \node[node style,font=\small] at (11.2, 0.7) {One-pass SGD \\ \cite{arous2021online}};
  \node[node style,font=\small] at (14.6, 0.7) {Kernel methods \\ \cite{ghorbani2019linearized}};

  % Academic references 
  \node[align=center, font=\normalsize] at (0, -0.41) {$d$};
  \node[align=center, font=\normalsize] at (1, -0.42) {$\tilde{\Theta}(d)$};
  \node[align=center, font=\normalsize] at (6.8, -0.42) {$\tilde{\Theta}(d^{p/2})$};
  \node[align=center, font=\normalsize] at (11.2, -0.42) {$\tilde{\Theta}(d^{p-1})$};
  \node[align=center, font=\normalsize] at (14.6, -0.42) {$\Theta(d^{q})$}; 
\end{tikzpicture}
\caption{\small Complexity of learning single-index model where the link function $\sigma_*$ is a degree-$q$ polynomial with information exponent $p$. For the CSQ lower bound, we translate the tolerance to sample complexity using the i.i.d.~concentration heuristic $\tau\approx n^{-1/2}$. 
We restrict ourselves to algorithms using polynomial compute; this excludes the sphere-covering procedure in \cite{damian2024computational} or exponential-width neural network in \cite{bach2017breaking,takakura2024mean}.
}
\label{fig:main}
\end{minipage}
% \vspace{-1mm}
\end{figure}

\section{Problem Setting and Prior Works}
\label{sec:setting}

\paragraph{Notations.}
$\|\cdot\|$ denotes the $\ell_2$ norm for vectors and the $\ell_2\to\ell_2$ operator norm for matrices. 
$O_d(\cdot)$ and $o_d(\cdot)$ stand for the big-O and little-o notations, where the subscript highlights the asymptotic variable $d$ and suppresses dependence on $p,q$; we write $\tilde{O}(\cdot)$ when (poly-)logarithmic factors are ignored. $\cO_{d,\P}(\cdot)$ (resp.~$o_{d,\P}(\cdot)$) represents big-O (resp.~little-o) in probability as $d\to\infty$. 
$\Omega(\cdot),\Theta(\cdot)$ are defined analogously. 
$\gamma$ is the standard Gaussian distribution in $\R$. 
We denote the $L^2$-norm of a function $f$ with respect to the data distribution (which will be specified) as $\norm{f}_{L^2}$. 
For $g:\R\to\R$, we denote $g^i$ as its $i$-th exponentiation, and $g^{(i)}$ is the $i$-th derivative. 
We say an event happens \textit{with high probability} when the failure probability is bounded by $\Exp{-C\log d}$ for large constant $C$.

\subsection{Complexity of Learning Single-index Models}
\label{sec:single-index}

We aim to learn a single-index model \eqref{eq:teacher} where the link function $\sigma_*:\R\to\R$ has information exponent $p$ defined as follows \cite{dudeja2018learning,arous2021online}.
\begin{defi}[Information exponent]
Let $\{\He_j\}_{j=0}^{\infty}$ denote the normalized Hermite polynomials. The information exponent of $g\in L^2(\gamma)$, denoted by $\mathrm{IE}(g):=p\in\N_+$, is the index of the first non-zero Hermite coefficient of $g$, that is, given $g(z) = \sum_{i=0}^\infty\alpha_i \He_i(z)$, $p := \min\{i \!>\! 0: \alpha_i \!\neq\! 0\}$.
\label{def:information-exponent}
\end{defi}
By definition, when $\sigma_*$ is a degree-$q$ polynomial, we always have $p \le q$. 
Note that $f_*$ contains $\Theta(d)$ parameters to be estimated, and hence \textit{information theoretically} $n\gtrsim d$ samples are both sufficient and necessary for learning \cite{mondelli2018fundamental,barbier2019optimal,damian2024computational}; however, the sample complexity achieved by different (polynomial time) algorithms depends on structure of the link function.  
\begin{itemize}[leftmargin=*,topsep=1mm,itemsep=1mm]
    \item \textbf{Kernel Methods.} Rotationally invariant kernels cannot adapt to the low-dimensional structure of single-index $f_*$ and hence suffer from the curse of dimensionality \cite{yehudai2019power,ghorbani2019linearized,donhauser2021rotational,ba2022high}. 
    By a standard dimension argument \citep{kamath2020approximate,hsu2021approximation,abbe2022merged}, we know that in the isotropic data setting, kernel methods (including neural networks in the lazy regime \cite{jacot2018neural,chizat2018note}) require $n\gtrsim d^q$ samples to learn degree-$q$ polynomials in $\R^d$. 
    \item \textbf{Gradient-based Training of NNs.} While NNs can easily approximate a single-index model \cite{bach2017breaking}, the sample complexity of gradient-based learning established in prior works typically scales as $n\gtrsim d^{\Theta(p)}$: in the well-specified setting, \cite{arous2021online} proved a sample complexity of $n=\tilde{\Theta}(d^{p-1})$ for online SGD, which is later improved to $\tilde{\Theta}(d^{p/2})$ by a smoothed objective \cite{damian2023smoothing}; as for the misspecified setting, \cite{bietti2022learning,dandi2023learning} showed that $n\gtrsim d^{p}$ samples suffice, and in some cases a $\tilde{\Theta}(d^{p-1})$ complexity is achievable \cite{abbe2023sgd,oko2024learning}. Consequently, at the information-theoretic limit ($n\asymp d$), existing results can only cover the learning of low information exponent targets \cite{abbe2022merged,berthier2023learning,ba2023learning}. 
    This exponential dependence on $p$ also appears in the CSQ lower bounds \cite{damian2022neural,abbe2022merged}, which is often considered to be indicative of the performance of SGD learning with the squared loss (see Section~\ref{sec:CSQ}). 
\end{itemize}

\paragraph{Statistical Query Learners.} 
If we do not restrict ourselves to correlational queries, the sample complexity of learning \eqref{eq:teacher} can be drastically improved. Specifically, for polynomial $\sigma_*$, \cite{chen2020learning} gave an SQ algorithm that achieves low generalization error in $n=\tilde{O}(d)$ samples, which is near the information-theoretic limit; the key ingredient is to construct nonlinear transformations to the labels that lowers the information exponent to $2$; similar preprocessing also appeared in context of phase retrieval \citep{mondelli2018fundamental,barbier2019optimal}. Such transformations do not belong to CSQ, but can be utilized by a full SQ learner to enhance the statistical efficiency. Recently, \cite{damian2024computational} introduced the \textit{generative exponent} which governs the complexity of SQ algorithms. 
\begin{defi}[Generative exponent] 
The generative exponent (GE) of $g\in L^2(\gamma)$ is defined as the lowest information exponent (IE) after arbitrary $L^2$ transformation, that is, 
\begin{align}
    p_* =: \mathrm{GE}(g) = \inf_{\cT\in L^2(P_y)} \mathrm{IE}(\cT\circ g). 
\end{align}
\end{defi}
The generative exponent is the smallest information exponent obtained by all possible label transformations. By definition we always have $p^*\le p$, and the gap between the two indices can be arbitrarily large; for example, for the Hermite polynomials we have $\mathrm{IE}(\He_k) = k$ whereas $\mathrm{GE}(\He_k)\le 2$. 

\cite{damian2024computational} established a sample complexity lower bound of $n = \Omega(d^{\nicefrac{p_*}{2}\vee 1})$ for full SQ learners with polynomial compute (assuming $\tau\approx n^{-1/2}$), and obtained matching upper bound by a tensor partial-trace algorithm. Our goal is to show that SGD training of two-layer neural network can also achieve a sample and runtime complexity that scales with $n \simeq d^{\Theta(p_*)}$, where the dimension dependence is governed by the generative exponent $p_*$ instead of the information exponent $p$.

\subsection{Can Gradient Descent Go Beyond Correlational Queries?}
\label{sec:CSQ}

\paragraph{Correlational statistical query.} 
A statistical query (SQ) learner \citep{kearns1998efficient,reyzin2020statistical}  accesses the target $f_*$ through noisy queries $\tilde{\phi}$ with error tolerance $\tau$: 
$|\tilde{\phi} - \E_{\vx,y}[\phi(\vx,y)]|\le\tau$. Lower bound on the performance of SQ algorithm 
is a classical measure of computational hardness. 
In the context of gradient-based optimization, an often-studied subclass of SQ is the \textit{correlational} statistical query (CSQ) \citep{bshouty2002using} where the query is restricted to (noisy version of) $\E_{\vx,y}[\phi(\vx)y]$. To see the connection between CSQ and SGD, consider the gradient of expected squared loss for one neuron $f_{\vw}(\vx)$: 
\begin{align}
    \nabla_{\vw} \E_{\vx,y}(f_{\vw}(\vx) - y)^2 
    \propto
    - \E_{\vx,y} [\underbrace{y \cdot \nabla_{\vw} f_{\vw}(\vx)}_{\text{correlational query}}] \,+\, \E_{\vx} [\underbrace{f_{\vw}(\vx)\cdot \nabla_{\vw} f_{\vw}(\vx)}_{\text{\!\!\!\!can be evaluated without $y$\!\!\!\!}}]. 
\end{align}
One can see that information of the target function is encoded in the correlation term in the gradient.  
To infer the statistical efficiency of GD in the empirical risk minimization setting, we replace the population gradient with the empirical average $\nabla_{\vw} (\frac{1}{n}\sum_{i=1}^n(f_{\vw}(\vx_i) - y_i)^2)$, and heuristically equate the CSQ tolerance $\tau$ with the scale of i.i.d.~concentration error $n^{-1/2}$. 

For the Gaussian single-index model class with information exponent $p$, \cite{damian2022neural} proved a lower bound stating that a CSQ learner either has access to queries with tolerance $\tau\lesssim d^{-p/4}$, or exponentially many queries are needed to learn $f_*$ with small population loss. Using the heuristic $\tau\approx n^{-1/2}$, this suggests a sample complexity lower bound $n\gtrsim d^{p/2}$ for polynomial time CSQ algorithm. This lower bound can be achieved by a landscape smoothing procedure \cite{damian2023smoothing} (in the well-specified setting), and is conjectured to be optimal for empirical risk minimization with SGD \cite{abbe2023sgd}.

\paragraph{SGD with reused data.} As previously discussed, the gap between SQ and CSQ algorithms primarily stems from the existence of label transformations that decrease the information exponent. While such transformation cannot be utilized by a CSQ learner, \cite{dandi2024benefits} argued that they may arise from two consecutive gradient updates using the same minibatch. For illustrative purposes, consider one neuron $f_{\vw}(\vx) = \sigma(\langle\vx,\vw\rangle)$ updated by two GD steps using the same data point $(\vx,y)$, starting from zero initialization $\vw^{0}=\mathbf{0}$ (we focus on the correlational term in the loss for simplicity): 
\begin{align}
    \!\!\!
    \vw^{2} =&~ \vw^{1} + \eta\cdot y\sigma'(\langle\vx,\vw^{1}\rangle) \vx =  \eta\sigma'(0)  \underbrace{y\cdot\vx}_{\text{\!\!\!CSQ term\!\!\!}} +~ \eta \underbrace{y\sigma'(\eta\sigma'(0) \norm{\vx}^2 \cdot y) \vx}_{\text{non-CSQ term}}. \label{eq:second-step} 
\end{align}
Under appropriate learning rate scaling $\eta\cdot\norm{\vx}^2=\Theta(1)$, one can see that in the second gradient step, the label $y$ is transformed by the nonlinearity $\sigma'$, even though the loss function itself is not modified.  
Based on this observation, \cite{dandi2024benefits} showed that if the non-CSQ term in \eqref{eq:second-step} reduces the information exponent to $1$, then \textit{weak recovery} (i.e., nontrivial overlap between the first-layer parameters $\vw$ and index features $\vtheta$) can be achieved after two GD steps with $n=\Theta(d)$ samples. 
% More details can be found in Section~\ref{sec:SGD-transformation}.

\subsection{Challenges in Establishing Statistical Guarantees}

Importantly, the analysis in \cite{dandi2024benefits} does not lead to concrete learnability guarantees for the class of single-index polynomials for the following reasons: 
$(i)$ it is not clear if an appropriate nonlinear transformation that lowers the information exponent can always be extracted from SGD with reused data, and $(ii)$ the weak recovery guarantee may not translate to a sample complexity for the trained NN to achieve small generalization error. 
We elaborate these technical challenges below. 

\paragraph{SGD decreases information exponent.} To show weak recovery, \cite[Definition 3.1]{dandi2024benefits} assumed that the student activation $\sigma$ can reduce the information exponent of the labels to $1$; 
while a few examples are given, the existence of such transformations in SGD is not guaranteed:  
\begin{itemize}[leftmargin=*,topsep=1mm,itemsep=1mm]
    \item The label transformation employed in prior SQ algorithms \cite{chen2020learning} is based on thresholding, which reduces the information exponent to $2$ for any polynomial $\sigma_*$; however, isolating such function from SGD updates on the squared loss is challenging. Instead, we make use of monomial transformation which can be extracted from SGD via Taylor expansion. 
    \item If the link function satisfies $p_*\ge 2$, its information exponent after arbitrary nonlinear transformation is at least $2$; such functions are predicted not be not learnable by SGD in the $n\asymp d$ regime \cite{dandi2024benefits}. To handle this setting, we analyze the SGD update up to $\mathrm{poly}(d)$ time, at which a nontrivial overlap can be established by a Grönwall-type argument similar to \cite{arous2021online}. For $p_* = 2$, this recovers results on phase retrieval when $\sigma_*(z) = z^2$ which requires $n = \Omega( d\log d)$ samples. 
\end{itemize} 

\paragraph{From weak recovery to sample complexity.} Note that weak recovery (i.e., $|\langle\vw,\vtheta\rangle| > \varepsilon$ for some small constant $\varepsilon>0$) is generally insufficient to establish low generalization error of the trained NN. Therefore, we need to show that starting from a nontrivial overlap, subsequent gradient steps can achieve \textit{strong recovery} of the index features (i.e., $|\langle\vw,\vtheta\rangle| > 1-\varepsilon$), despite the link misspecification. After the first-layer parameters align with the target function, we train the second-layer parameters with SGD to learn the link function $\sigma_*$ with the aid of random bias units \cite{damian2022neural}.

\section{Learning Polynomial $f_*$ in Linear Sample Complexity}
\label{label:result}

We first consider the setting where $\sigma_*$ is polynomial with degree $q$ specified as follows. 
\begin{assumption}
The target function is given as $f_*(\vx) = \sigma_*(\langle\vx,\vtheta\rangle)$, where the link function $\sigma_*:\R\to\R$ admits the Hermite decomposition    $\sigma_*(z) = \sum_{i=p}^q \alpha_i\He_i(z)$. 
\label{assump:teacher}
\end{assumption}
For single-index polynomials, we do not expect a computational-to-statistical gap under the SQ class \cite{chen2020learning} --- indeed, we will establish learning guarantees near the information theoretic limit $n\asymp d$. 

\subsection{Training Algorithm}

We train the following two-layer network with $N$ neurons using SGD to minimize the squared loss: 
\begin{align}
    f_{\vTheta}(\vx) = \frac{1}{N}\sum_{j=1}^N a_j\sigma_j(\langle\vx,\vw_j\rangle + b_j), 
    \label{eq:student}
\end{align}
where $\vTheta=(\vw_j,a_j,b_j)_{j=1}^N$ are trainable parameters, and $\sigma_j:\R\to\R$ is the activation function defined as the sum of Hermite polynomials up to degree $C_\sigma$: $\sigma_j(z):= \sum_{i=0}^{C_\sigma}\beta_{j,i} \He_i(z)$, where $C_\sigma$ only depends on the degree of link function $\sigma_*$. Note that we allow each neuron to have a different nonlinearity as indicated by the subscript in $\sigma_j$; this subscript is omitted when we focus on the dynamics of one single neuron. 
Our SGD training procedure is described in Algorithm~\ref{alg:main}, and below we outline the key ingredients of the algorithm. 
\begin{itemize}[leftmargin=*,topsep=1mm,itemsep=1mm]
    \item Algorithm~\ref{alg:main} employs a layer-wise training strategy common in the recent feature learning theory literature \cite{damian2022neural,ba2022high,bietti2022learning,abbe2023sgd,mousavi2023gradient}, where in the first stage, we optimize the first-layer parameters $\{\vw_j\}_{j=1}^N$ with normalized SGD to learn the low-dimensional latent representation (index features $\vtheta$), and in the second phase, we train the second-layer $\{a_j\}_{j=1}^N$ to fit the unknown link function $\sigma_*$.
    \item The most crucial part in Phase I of Algorithm~\ref{alg:main} is the reuse of the same minibatch in the gradient computation. Specifically, we sample a fresh batch of training examples in \textit{every two GD steps}; this enables us to extract non-CSQ terms from two consecutive gradient updates outlined in \eqref{eq:second-step}. 
    \item We introduce an \textit{interpolation step} between the current and previous iterates with hyperparameter $\xi$ 
    to stabilize the training dynamics; this resembles a negative momentum often seen in optimization algorithms \cite{allen2018katyusha,zhang2019lookahead}; the role of this interpolation is discussed in Section~\ref{sec:SGD-transformation}. We use a projected gradient update $\tilde{\nabla}_{\vw} \cL(\vw) = (\vI_d-\vw^{2t}{\vw^{2t}}^\top)\nabla_{\vw}\cL(\vw)$ for steps $2t$ and $2t+1$, where $\nabla_{\vw}$ is the Euclidean gradient; similar use of projection also appeared in \citep{damian2023smoothing,abbe2023sgd}.
\end{itemize}

\begin{algorithm}[t]
\SetKwInOut{Input}{Input}
\SetKwInOut{Output}{Output}
\SetKwBlock{StageOne}{Phase I: normalized SGD on first-layer parameters}{}
\SetKwBlock{StageTwo}{Phase II: SGD on second-layer parameters}{}
\SetKw{Initialize}{Initialize}
	
\Input{Step sizes $\eta^t$; momentum parameters $\xi^t$; training time $T_1,T_2$; $\ell_2$ regularization $\lambda$. }

\Initialize{$\vw^0_j\sim\mathbb{S}^{d-1}(1)$, $a_j \sim \mathrm{Unif}\{\pm c_a\}$. }
% $a_j \sim \mathrm{Unif}\{\pm 1\}$. }

\StageOne{
\For{$t=0$ {\bfseries to} $ T_1$}{
   \If{\underline{$t$ is even}}{
   $\vx\sim\cN(0,\vI_d),~ y = f_*(\vx) + \varsigma$ \tcp*{{\color{blue!60!gray} Draw i.i.d.~data $(\vx,y)$}} 
   ${\vw}^t_j \leftarrow \vw^{t}_j-\xi_j^t (\vw^{t}_j-\vw^{t-2}_j)$, ~ (when $t>0$) \tcp*{{\color{blue!60!gray}Interpolation step}} 
   $\vw^{t}_j \leftarrow \vw^{t}_j / \|\vw^{t}_j\|$ \tcp*{{\color{blue!60!gray}Normalization}} 
   }
   $\vw^{t+1}_j \leftarrow \vw^t_j- \eta^t
   \tilde{\nabla}_{\vw}(f_{\vTheta}(\vx) -y)^2 , ~~  (j=1,\dots,N)$ \tcp*{{\color{blue!60!gray} {SGD step}}} 
   }
}	
\Initialize{$b_j\sim \mathrm{Unif}([-C_b,C_b])$. }

\StageTwo{
$
   \hat{\va} \leftarrow \mathrm{argmin}_{\va\in\R^N}\frac{1}{T_2}\sum_{i=1}^{T_2} \left(f_{\vTheta}(\vx_i)-y_i\right)^2 + \lambda\|\va\|^2
$
\tcp*{{\color{blue!60!gray}Ridge regression}} 
	}
 
\Output{Prediction function $\vx\mapsto f_{\hat{\vTheta}}(\vx)$ with $\hat{\vTheta}=(\hat{a}_j,\vw^{T_1}_j,b_j)_{j=1}^{N}$.}
\caption{Gradient-based training of two-layer neural network} 
 \label{alg:main}
\end{algorithm}

\subsection{Convergence and Sample Complexity}
\label{subsec:weak-strong-recovery}

\paragraph{Weak Recovery Guarantee.} We first consider the ``search phase'' of SGD, and show that after running Phase I of Algorithm~\ref{alg:main} for $T=\mathrm{polylog}(d)$ steps, a subset of parameters $\vw$ achieve nontrivial overlap with the target direction $\vtheta$. We denote $H(g;j)$ as the $j$-th Hermite coefficient of some $g\in L^2(\gamma)$. 
Our main theorems handle polynomial activations satisfying the following condition.

\begin{assumption}
We require the activation function to be a polynomial $\sigma(z)= \sum_{i=0}^{C_\sigma}\beta_i \He_i(z)$ and its degree $C_\sigma$ to be sufficiently large so that $C_\sigma\geq C_q$ holds ($C_q$ is defined in Proposition~\ref{proposition:Reduction}).
For all $2\leq \ell\leq C_\sigma$ and $k=0,1$, we assume that $ H\big(\sigma^{(\ell)}(\sigma^{(1)})^{\ell-1};k\big) > 0$.
\label{assump:student1}
\end{assumption}
As discussed in Appendix~\ref{subsection:linkfunction}, for a given $\sigma_*$, the above assumption only needs to be met for one pair of $(k,\ell)$.
Appendix~\ref{subsubsection:Moredicsussionon} states that $ H\big(\sigma^{(\ell)}(\sigma^{(1)})^{\ell-1};k\big) \ne 0$ also suffices if we set the momentum parameter $\xi$ differently.
Now we verify this condition for a wide range of polynomial activations. 
\begin{lemma}\label{lemm:non-zero-assumption}
    Given $\ell\ge 2$ and $k\ge 0$. For $C_\sigma\ge \frac{2\ell+k-1}{\ell}$, if we choose $\{\beta_i\}_{i=0}^{C_\sigma}$ where $\beta_i$ is randomly drawn from some non-empty interval $[a_i,b_i]$, 
   then $H(\sigma^{(\ell)}(\sigma^{(1)})^{\ell-1};k) \ne 0$ with probability 1.
\end{lemma}

The next theorem states that $n=\tilde{\Theta}(d)$ samples are sufficient for SGD to achieve weak recovery. 
\begin{theorem}\label{theorem:WeakRecovery}
Under Assumptions \ref{assump:teacher} and \ref{assump:student1}, for suitable choices of hyperparameters $\eta^t = \tilde{O}_d(Nd^{-1})$ and $1-\xi^t = o_d(1)$, there exists constant $C(q)$ such that after Phase I of Algorithm~\ref{alg:main} is run for $2T_{1,1}=C(q)\cdot d\mathrm{polylog}(d)$ steps, with high probability, there exists a subset of neurons $\vw_j^{2T_1}\in\cW$ with $|\cW| = \tilde{\Theta}(N)$ such that $\big|\langle\vw_j^{2T_1},\vtheta\rangle\big| > c$ for some $c\gtrsim 1/\mathrm{polylog}(d)$. 
\label{thm:weak-recovery}
\end{theorem}
Recall that at random initialization we have $\langle\vw,\vtheta\rangle \approx d^{-1/2}$ with high probability. The theorem hence implies that SGD ``escapes from mediocrity'' after seeing $n=\tilde{O}(d)$ samples, analogous to the information exponent $p=2$ setting studied in \cite{arous2021online}. We remark that due to the small second-layer initialization, the squared loss is dominated by the correlation loss, which allows us to track the evolution of each neuron independently; similar use of vanishing initialization also appeared in \cite{ba2022high,abbe2023sgd}. 
% This will be formally proved in Appendix~\ref{subsection:Stochastic}.

\paragraph{Strong recovery and sample complexity.}
After weak recovery is achieved, we continue Phase I to amplify the alignment. Due to the nontrivial overlap between $\vw$ and $\vtheta$, the objective is no longer dominated by the lowest degree in the Hermite expansion. Therefore, to establish strong recovery ($\langle\vw,\vtheta\rangle>1-\varepsilon$), we place an additional assumption on the activation function. 
\begin{assumption}
\label{assump:student2}  
Given the Hermite expansions $\sigma_*(z) = \sum_{i=p}^q \alpha_i\He_i(z)$, $\sigma_j(z)= \sum_{i=0}^{C_\sigma}\beta_{j,i} \He_i(z)$, we assume the coefficients satisfy $\alpha_i\beta_{j,i}\ge 0$ for $p\le i\le q$. 
\end{assumption}
This assumption is easily verified in the well-specified setting $\sigma_*=\sigma$ \cite{arous2021online} since $\alpha_i=\beta_i$, and under link misspecification, it has been directly assumed in prior work \cite{mousavi2023gradient}. We follow \cite{oko2024learning} and show that by randomizing the Hermite coefficients of the activation function, a subset of neurons satisfy the above assumption for any degree-$q$ polynomial link function $\sigma_*$. 
\begin{lemma}\label{lemma:atleastafewneurons}
If we set $\sigma_j(z)=\sum_{i=0}^{C_\sigma} \beta_{j,i} \He_i(z)$, where for each neuron we sample $\beta_{j,i} \iid \mathrm{Unif}(\{\pm r_i\})$ with appropriate constant $r_i$, then Assumption~\ref{assump:student1} and \ref{assump:student2} are satisfied in $\Exp{-\Theta(q)}$-fraction of neurons. 
\label{lemm:randomized-activation}
\end{lemma}
Note that in our construction of activation functions for both assumptions, we do not exploit knowledge of the link function $\sigma_*$ other than its degree $q$ which decides the constant $C_\sigma$; see Appendix~\ref{subsection:linkfunction} for more discussion of Assumption~\ref{assump:student2} and Lemma~\ref{lemma:atleastafewneurons}.
The next theorem shows that by running Phase I for $\tilde{\Theta}(d)$ more steps, a subset of neurons achieves sufficiently large overlap with the index features.  
\begin{theorem}
For student neurons satisfying Assumptions \ref{assump:student1}, \ref{assump:student2} and parameter $\vw_j$ 
starting from nontrivial overlap $c>0$ specified in Theorem~\ref{thm:weak-recovery}, if Phase I of Algorithm~\ref{alg:main} continues for $2T_{1,2}=\tilde{\Theta}_d(d\varepsilon^{-2})$ steps with hyperparameters $\eta^t = \tilde{O}_d(Nd^{-1}\varepsilon)$, $\xi^t = 1$, we achieve $\big\langle\vw_j^{2(T_{1,1}+T_{1,2})},\vtheta\big\rangle > 1-\varepsilon$ with high probability. 
\label{thm:strong-recovery} 
\end{theorem}

The following proposition shows that after strong recovery, training the second-layer parameters in Phase II is sufficient for the NN model \eqref{eq:student} to achieve small generalization error. 
\begin{proposition}
After Phase I terminates, for suitable $\lambda > 0$, the output of Phase II satisfies
$$
\mathbb{E}_{\vx}[(f_{\hat{\vTheta}}(\vx) - f_*(\vx))^2] 
\lesssim \varepsilon^2. 
$$
with probability 1 as $d\to\infty$, if we set 
$T_2=C(q)N^{4}\mathrm{polylog}(d)\varepsilon^{-4}$, $N = C(q)\mathrm{polylog}(d) \varepsilon^{-1}$ for some constant $C(q)$ depending on the target degree $q$. 
\label{prop:2nd-layer-training} 
\end{proposition}

\paragraph{Putting things together.} Combining the above theorems, we conclude that in order for two-layer NN \eqref{eq:student} trained by Algorithm~\ref{alg:main} to achieve $\varepsilon$ population squared loss, it is sufficient to set
\begin{align}
    n = T_1 + T_2 \asymp  C(q)\cdot (d\varepsilon^{-2}\vee\varepsilon^{-8}) \mathrm{polylog}(d), 
    \quad N \asymp C(q)\cdot \varepsilon^{-1}\mathrm{polylog}(d),
\end{align}
where constant $C(q)$ only depends on the target degree $q$ (although exponentially). Hence we may set $\varepsilon^{-1}\asymp \mathrm{polylog}d$ to conclude an almost-linear sample and computational complexity for learning arbitrary single-index polynomials up to $o_d(1)$ population error. 
% This establishes the informal theorem in Section~\ref{sec:intro}. 

\section{Proof Sketch}

In this section we outline the high-level ideas and key steps in our derivation. 

\subsection{Monomial Transformation Reduces Information Exponent} 
\label{sec:monimial-transformation}

To prove the main theorem, we first establish the existence of nonlinear label transformation that $(i)$ reduces the information exponent, and $(ii)$ can be easily extracted from SGD updates. If we ignore desideratum $(ii)$, then for polynomial link functions, transformations that decrease the information exponent to at most $2$ have been constructed in \cite[Section 2.1]{chen2020learning}. However, prior results are based on the thresholding function, and it is not clear if such function naturally arises from SGD with batch reuse. 
The following proposition shows that the effect of thresholding can also be achieved by a simple monomial transformation where the required degree can be uniformly upper bounded. 

\begin{proposition}\label{proposition:Reduction}    
    Let $g:\R\to\R$ be any polynomial with degree up to $p$ and $\norm{g}_{L^2(\gamma)}^2=1$, then 
    \begin{itemize}[leftmargin=7mm,topsep=1mm,itemsep=1mm]
        \item[(i)] There exists some $i\leq C_q\in\N_+$ such that ${\mathrm{IE}}(g^i) \leq 2$, where constant $C_q$ only depends on $q$.
        \item[(ii)] Let $g^{\mathrm{odd}}:\R\to\R$ be the odd part of $g$ with $\|g^{\mathrm{odd}}\|_{L^2(\gamma)}^2 \geq \rho>0$.
        Then there exists some $i\leq C_{q,\rho}\in\N_+$ such that ${\mathrm{IE}}(g^i) =1$, where constant $C_{q,\rho}$ only depends on $q$ and $\rho$.
    \end{itemize}
    \label{prop:monomial-transformation}
\end{proposition}
The proof can be found in Appendix~\ref{app:transformation}. 
We make the following remarks. 
\begin{itemize}[leftmargin=*]
    \item The proposition implies that for any polynomial link function that is not even, there exists $i\in\N_+$ only depending on the degree of $\sigma_*$ such that raising the function to the $i$-th power reduces the information exponent to $1$ ($p_*=1$). For even $\sigma_*$, the information exponent after arbitrary transformation is at least $2$ ($p_*=2$), which can also be attained by monomial transformation. Furthermore, we provide a \textit{uniform} upper-bound on the required degree of transformation $i$ via a compactness argument. 
    \item The advantage of working with monomial transformations is that they can be obtained from two GD steps on the same training example, by Taylor expanding the activation $\sigma'$. In Section \ref{sec:SGD-transformation}, we build upon this observation to show that Phase I of Algorithm \ref{alg:main} achieves weak recovery using $n \gtrsim d\,\mathrm{polylog}(d)$ samples.  
\end{itemize}

\paragraph{Intuition behind the analysis.} 
Our proof is inspired by \cite{chen2020learning} which introduced a (non-polynomial) label transformation that reduces the information exponent of any degree-$q$ polynomial to at most $2$. 
To prove the existence of monomial transformation for the same purpose, we first show that for a fixed link function $\sigma_*$, there exists some $i$ such that the $i$-th power of the link function has information exponent $2$, which mirrors the transformation used in \cite{chen2020learning}. Then, we make use of the compactness of the space of link functions to define a test function and obtain a uniform bound on $i$. As for the polynomial transformation for non-even functions, we exploit the asymmetry of $\sigma_*$ to further reduce the information exponent to 1.

\subsection{SGD with Batch Reuse Implements Polynomial Transformation} 
\label{sec:SGD-transformation}
Now we present a more formal discussion of \eqref{eq:second-step} to illustrate how polynomial transformation can be utilized in batch reuse SGD. 
We let $\eta^t \equiv \eta$. 
% Here we let ignore polylogarithmic factors.
When one neuron $f_{\vw}(\vx) = \sigma(\langle\vx,\vw\rangle)$ is updated by two GD steps using the same sample $(\vx,y)$, starting from $\vw^{0}:=\vomega$, the alignment with $\vtheta$ becomes
\begin{align}
  &\langle\vtheta,\vw^{2}\rangle 
  = \left\langle\vtheta,\big[\vw^{1} + \eta\cdot y\sigma'(\langle\vx,\vw^{1}\rangle) \vx\big]\right\rangle
 % \\ & =   \vtheta^\top \big[\vw^{(0)}+\eta y\sigma'(\vomega^\top \vx) \vx +~ \eta y\sigma'(\vomega^\top \vx+\eta \norm{\vx}^2 
 % y\sigma'(\vomega^\top \vx) ) \vx\big].
 = \langle\vtheta,\vomega\rangle + \\
 & \eta \bigg[\textstyle y\sigma'(\langle\vomega,\vx\rangle) \langle\vtheta,\vx\rangle
 + \sum_{i=0}^{C_\sigma-1} \underbrace{\textstyle(\eta \|\vx\|^2)^{i}y^{i+1} (i!)^{-1}(\sigma'(\langle\vomega,\vx\rangle))^{i}\sigma^{(i+1)}(\langle\vomega,\vx\rangle)\langle\vtheta,\vx\rangle}_{=:\psi_i}\bigg]
   \label{eq:two-steps-2}.
\end{align}
We take $\eta \leq c_\eta d^{-1}$ with a small constant $c_\eta$ so that $\eta \|\vx\|^2 \ll 1$ with high probability. 
Crucially, the strength of each term in \eqref{eq:two-steps-2} can vary depending on properties of the unknown link function $\sigma_*$. 
Hence a careful analysis is required to ensure that a suitable monomial transformation is always singled out from the gradient. 
We establish the following lemma on the evolution of alignment. 
\begin{lemma}
Under the assumptions per Theorem~\ref{theorem:WeakRecovery}, the following holds for generative exponent $p_*=1,2$:
\begin{align}
      \langle\vtheta,\vw^{2(t+1)}\rangle \geq \langle\vtheta,\vw^{2t}\rangle + c_\eta^I  c_{\xi}c_\sigma d^{-\frac{p_*}{2}\lor 1}(\kappa^{2t})^{p_*-1}
        + c_\eta  c_{\xi} d^{-\frac{p_*}{2}\lor 1}\nu^{2t}.
\end{align}
\end{lemma}
See Lemma~\ref{lemma:PopulationPoly} for the formal version. 
For $p_*=1$, taking expectation immediately yields that weak recovery within $(\eta (1-\xi) \gamma)^{-1}=O(d)$ steps. 
For $p_*=2$, $\langle\vtheta,\vw_j^{2t}\rangle=:\kappa^t$ can be approximated by a differential equation $\frac{\mathrm{d}\kappa^t}{\mathrm{d}t} = \eta(1-\xi) \gamma \kappa^t$.
Solving this yields $\kappa^t = \kappa^0 \exp(\eta(1-\xi) \gamma t) \approx d^{-\frac12}\exp(\eta(1-\xi) \gamma t)$, and weak recovery is obtained within $t \lesssim (\eta (1-\xi) \gamma)^{-1} \cdot \log d = O(d\log d)$ steps, similar to the analysis in \cite{arous2021online}.

\paragraph{Why interpolation is needed.}
In our setting, the signal strength may not dominate the error from discarding the effect of normalization. In prior analyses for online SGD, given the gradient $-\vg$ and projection $P_{\vw}=\vI_d-\vw\vw^\top$, 
the spherical gradient changes the alignment as 
$\langle\vtheta,\vw^{t+1}\rangle=\big\langle\vtheta,\frac{ \vw^t + \eta P_{\vw} \vg}{\|\vw^t+\eta P_{\vw}\vg\|}\big\rangle\geq \langle\vtheta,\vw^t\rangle + \eta \langle\vtheta,\vg\rangle -\frac12 \eta^2 \|\vg\|^2 \langle\vtheta, \vw^t\rangle+\text{(negligible terms)}$, see \citep{arous2021online,damian2023smoothing}. 
Here $\eta \langle\vtheta,\vg\rangle$ corresponds to the signal, and $-\frac12 \eta^2 \|\vg\|^2 \langle\vtheta, \vw^t\rangle$ comes from normalization. 
Thus, taking $\eta$ sufficiently small, the normalization error shrinks faster than the signal.
However, in our case the signal shrinks at the rate of $c_\eta^I$ (recall that $\eta = c_\eta d^{-1}$), and hence taking a smaller step may not improve the signal-to-noise ratio when the degree of transformation $I$ is large. 
The interpolation step in Algorithm~\ref{alg:main} reduces the effect of normalization without shrinking the signal too much, by ensuring $\vw^{2(t+1)}$ stays close to $\vw^{2t}$. In particular, by setting $\xi = 1-\tilde{\eta}$, we see that the signal is affected by a factor of $\tilde{\eta}$ whereas the normalization error shrinks by $\tilde{\eta}^2$; this allows us to boost the signal-to-noise ratio by taking $\tilde{\eta}$  small.

\subsection{Analysis of Phase II and Statistical Guarantees}

Once strong recovery is achieved for the first-layer parameters, we turn to Phase II and optimize the second-layer with $\ell_2$ regularization. Since the objective is strongly convex, gradient-based optimization can efficiently minimize the empirical loss. 
In Appendix~\ref{subsection:Second}, the learnability guarantee follows from standard analysis analogous to that in \cite{abbe2022merged,damian2022neural,ba2022high}, where we construct a ``certificate'' second-layer $\va^*\in\R^N$ that achieves small loss and small norm: 
$$
\textstyle \E_{\vx}\left(f_*(\vx) - \frac{1}{N}\sum_{j=1}^N a^*_j\sigma_j\big(\langle{\vw_j^{T_1}}, \vx\rangle+b_j\big)\right)^2 \le \varepsilon^*,\quad \|\va^*\| \lesssim r^*, 
$$
from which the population loss of the regularized empirical risk minimizer can be bounded via standard Rademacher complexity argument. 
To construct such a certificate, we make use of the random bias units $\{b_j\}_{j=1}^N$ to approximate the link function $\sigma_*$ as done in \cite{damian2022neural,bietti2022learning,oko2024learning}.

\section{Beyond Polynomial Link Functions}

Thus far we have shown that for polynomial single-index target functions (which satisfy $p_*\le 2$), SGD with data reuse can implement a polynomial transformation to the labels that reduces the information exponent to at most 2; consequently, the trained two-layer neural network can achieve small generalization error with $n = d\polylog(d)$ samples. 
However, as shown in \cite{damian2024computational}, there exists (non-polynomial) $\sigma_*$ with generative exponent $p_*>2$ (i.e., label transformations cannot lower the information exponent to $2$) and thus not learnable by SQ algorithms in linear sample complexity. 

Nevertheless, for a single-index model with generative exponent $p_*$, we know there exists an ``optimal'' label transformation that reduces the information exponent to $p_*$. If SGD can make use of such transformation, then from the arguments in \cite{arous2021online}, it is natural to conjecture that a sample size of $n\simeq d^{p_*-1}$ is sufficient. In this section we confirm this intuition by proving that SGD with data reuse (Algorithm~\ref{alg:main}) indeed matches this complexity. The following lemma is an analogue of Proposition~\ref{prop:monomial-transformation} stating that polynomial transformations are sufficient to lower the information exponent. 
\begin{lemma}\label{lemm:Reduction_Generative}
Given link function $\sigma_*$ with generative exponent $p_*\in\N_+$.  
Suppose we can take an orthonormal polynomial basis $\{\phi_k\}_k$ for the space $L^2(P_y)$ with inner product $\langle f,g\rangle=\mathbb{E}_{y=\sigma_*(z)}[f(y)g(y)]$. 
Then there exists some degree of transformation $I\in\N_*$ such that $\mathrm{IE}(\sigma_*^I)=p_*$. 
\end{lemma} 

We outline the differences and additional technical challenges to handle the $\mathrm{GE}(\sigma_*)>2$ setting. 
\begin{itemize}[leftmargin=*,topsep=1mm,itemsep=1mm]
    \item For general $L^2$ link functions $\sigma_*$, we can no longer make use of the compactness argument (see proof of Proposition~\ref{prop:monomial-transformation}) to upper bound the degree of monomial transformation. Hence in Lemma~\ref{lemm:Reduction_Generative} we do not state a uniform upper bound on the required degree $I$, unlike the polynomial setting. 
    \item Any link function with $p_*>2$ cannot be polynomial, and hence we cannot achieve low generalization error using a neural network with polynomial nonlinearity. We therefore need to use an activation function with universal function approximation ability. 
\end{itemize}

\subsection{Sample Complexity for Weak Recovery}

We first show that Algorithm~\ref{alg:main} achieves weak recovery (i.e., nontrivial overlap with the ground truth $\vtheta$) with a complexity governed by the generative exponent of the link function $p_* = \mathrm{GE}(\sigma_*)$. Similar to Section~\ref{subsec:weak-strong-recovery}, we make use of randomized activation functions to ensure the desired label transformation is encoded --- we defer the conditions on the student activation 
to Appendix~\ref{subsubsection:Assumption-GeneralLink}. 
Similar to Theorem~\ref{theorem:WeakRecovery}, we focus on the subset of neurons with large initial overlap, and activation satisfying the assumptions in Appendix~\ref{subsubsection:Assumption-GeneralLink} (these conditions are met by $\tilde{\Omega}(1)$ fraction of neurons). 
\begin{prop}\label{prop:StochasticGE}

        Suppose the link function $\sigma_*$ has generative exponent $p_*$, and let $I\in\N_+$ be the smallest degree of monomial transformation that lowers the information exponent to $p_*$ (i.e., $\mathrm{IE}(\sigma_*^I) = p$). We can find a student activation function $\sigma$ depending only on $p,p_*$ and $I$, such that if we take $\eta^{2t},\eta^{2t+1} = c_\eta N d^{-1}$, $\xi^{2(t+1)} = 1-c_{\xi}d^{-\nicefrac{(p_*-2)_+}{2}}$ for small $c_\eta, c_\xi= o_d(1)$, and set
        \begin{align}
            T_{1,1} \simeq c_\xi^{-1} 
            \begin{cases}
                d & (\text{if $p_*=1$})
                \\ 
                d (\log d) &(\text{if $p_*=2$})
                \\
                d^{p_*-1} &(\text{if $p_* \geq 3$}),
            \end{cases}
        \end{align}
        then if the initial alignment $\langle\vw^0,\vtheta\rangle\geq 2c_\eta^{-1}d^{-\nicefrac{1}{2}}$, 
        there exists $\tau_*\leq T_{1,1}$ such that for all $\tau\ge\tau_*$,
        \begin{align}
            \langle\vw^{2\tau},\vtheta\rangle \geq \tilde{\Theta}(1), \quad \text{with probability $1-o_d(1)$.}
        \end{align}
\end{prop}
Proposition~\ref{prop:StochasticGE} is a generalization of Theorem~\ref{thm:weak-recovery} beyond polynomial $\sigma_*$ (the proof of both results are presented in Appendix~\ref{subsection:Expected},\ref{subsection:Stochastic}), and can be interpreted as an SQ counterpart to \cite{arous2021online}: we establish a sufficient sample size of $n\simeq d^{(p_*-1)\vee 1}$ for Algorithm~\ref{alg:main} to exit the search phase, which is parallel to the $n\simeq d^{(p-1)\vee 1}$ rate for one-pass SGD (note that our rates are slightly sharper due to logarithmic factors removed, since $c_\xi^{-1}$ can grow arbitrarily slowly with $d$). For high generative exponent $\sigma_*$ with $p_*>2$, we no longer match the information theoretically optimal sample complexity $n\asymp d$, which is consistent with the computational-to-statistical gap observed in \cite{damian2024computational}. 

\subsection{Generalization Error Guarantee}

After Phase I of Algorithm~\ref{alg:main}, we learn the unknown link function $\sigma_*$ via training the second-layer. To approximate non-polynomial functions, we introduce a ReLU component in the student nonlinearity $\sigma$ (see Lemma \ref{lemma:DesignActivation-2} for discussions), and make use of the approximation result for the (univariate) ReLU kernel in \cite{bietti2022learning}, which handles general $\sigma_*$ whose second derivative has bounded 4th moment. 
Combining the above, we arrive at the following end-to-end guarantee for learning single-index models with arbitrary generative exponent using SGD training of neural network. 

\begin{prop}[Informal] 
\label{prop:high-GE-test-error}
Suppose the link function $\sigma_*$ has generative exponent $p_*\in\N_*$ and satisfies $\sigma_*,\sigma_*''\in L^4(\gamma)$. For appropriately chosen activation function $\sigma$ (see Appendix~\ref{subsubsection:Assumption-GeneralLink}), a neural network \eqref{eq:student} with $N=\tilde{\Theta}(1)$ neurons optimized by Algorithm~\ref{alg:main} achieves small population loss $\mathbb{E}_{\vx}[(f_{\hat{\vTheta}}(\vx) - f_*(\vx))^2] = o_{d,\P}(1)$ with a sample complexity of $n = \tilde{\Theta}(d^{(p_*-1)\vee 1})$.  
\end{prop}
See Appendix~\ref{subsection:Second} for the full statement with $\varepsilon$ dependence. This proposition confirms that the sample complexity for weak recovery (Proposition~\ref{prop:StochasticGE}) is the bottleneck in single-index learning, as the total sample size required for Algorithm~\ref{alg:main} to achieve low test error also scales with $d^{(p_*-1)\vee 1}$.

\section{Conclusion and Future Directions}
% \vspace{-0.5mm}  

We showed that a two-layer neural network \eqref{eq:student} trained by SGD with reused batch can learn single-index model (with generative exponent $p_*$) using $n\simeq d^{(p_*-1)\vee 1}$ samples and compute; in particular, when the link function $\sigma_*$ is polynomial, we established a sample complexity of $n=\tilde{O}(d\varepsilon^{-2})$ to achieve $\varepsilon$ population loss, which is almost information theoretically optimal. Our analysis is based on the observation that by reusing the same training data twice in the gradient computation, a non-correlational term arises in the SGD update that transforms the labels (despite the loss function not modified). We proved that monomial transformations that lower the information exponent of $\sigma_*$ can be extracted by Taylor-expanding the SGD update; then we showed via careful analysis of the trajectory that strong recovery and low population loss is achieved under suitable activation function. 

Interesting future directions include extension to multi-index models \cite{ben2022high,bietti2023learning,collins2023hitting}, hierarchical target functions \cite{allen2019can,nichani2023provable}, and in-context learning \cite{oko2024pretrained}. Also, the SGD algorithm that we employ requires a layer-wise training procedure and a specific batch reuse schedule; one may therefore ask if standard multi-pass SGD training of all parameters simultaneously \cite{glasgow2023sgd} (as reported in Figure~\ref{fig:intro}) also achieves the same statistical efficiency.

% \bigskip

\subsection*{Acknowledgements}

The authors thank Gerard Ben Arous, Joan Bruna, Alex Damian, Marco Mondelli, and Eshaan Nichani for the discussions and feedback on the manuscript. JDL acknowledges support of the ARO under MURI Award W911NF-11-1-0304, NSF CCF 2002272, NSF IIS 2107304,  NSF CIF 2212262, ONR Young Investigator Award, and NSF CAREER Award 2144994. 
KO was partially supported by JST ACT-X (JPMJAX23C4).
TS was partially supported by JSPS KAKENHI (24K02905) and JST CREST (JPMJCR2015). 
This research is unrelated to DW's work at xAI.

\bigskip

{

\fontsize{10}{11}\selectfont     

\bibliography{citation}
\bibliographystyle{alpha}

} 

\newpage
{
\renewcommand{\contentsname}{Table of Contents}
\tableofcontents
}

\newpage

\appendix
\renewcommand{\cA}{c_1}
\renewcommand{\cB}{c_2}
\renewcommand{\cC}{c_1}
\renewcommand{\cD}{c_5}
\newcommand{\CA}{C_3}
\newcommand{\CB}{C_1}
\newcommand{\CC}{C_2}
\renewcommand{\CD}{C_4}

\allowdisplaybreaks

\section{Polynomial Transformation}
\label{app:transformation}

\begin{proofof}[Proposition~\ref{prop:monomial-transformation}] 
We use a thresholding and compactness argument inspired by \cite{chen2020learning}. 

\subsection{Proof for Even Functions $(i)$}

We divide the analysis into the following steps. 

    \noindent {\bf (i-1): Monomials reducing the information exponent.}
    Define $\tau(f)=\max_{-2\leq t\leq 2}|f(t)|$.
    This entails that if $|f(t)|\geq \tau(f)$, then we have $|t|>2$.
    
    Consider the following expectation:
    \begin{align}
        \mathbb{E}_{t\sim \mathcal{N}(0,1)}\bigg[\bigg(\frac{f(t)}{2\tau(f)}\bigg)^{i} (t^2-1)\bigg].
        \label{eq:PT-1}
    \end{align}
    We evaluate the case when $i$ is even.
    \eqref{eq:PT-1} can be lower bounded as
    \begin{align}
      \eqref{eq:PT-1}
     = &~  \mathbb{E}_{t\sim \mathcal{N}(0,1)}\bigg[\mathbbm{1}[|f(t)|\geq 2\tau(f)]\bigg(\frac{f(t)}{2\tau(f)}\bigg)^{i} (t^2-1)\bigg] \\
        &\quad +\mathbb{E}_{t\sim \mathcal{N}(0,1)}\bigg[\mathbbm{1}[\tau(f)\leq |f(t)|<2 \tau(f)]\bigg(\frac{f(t)}{2\tau(f)}\bigg)^{i} (t^2-1)\bigg]
      \\ & \quad  +\mathbb{E}_{t\sim \mathcal{N}(0,1)}\bigg[\mathbbm{1}[|f(t)|< \tau(f)]\bigg(\frac{f(t)}{2\tau(f)}\bigg)^{i} (t^2-1)\bigg]
      \\ \geq &~ \mathbb{E}_{t\sim \mathcal{N}(0,1)}\bigg[\mathbbm{1}[|f(t)|\geq 2\tau(f)]\bigg(\frac{2\tau(f)}{2\tau(f)}\bigg)^{i}(2^2-1)\bigg] \\
      &\quad  + \mathbb{E}_{t\sim \mathcal{N}(0,1)}\bigg[\mathbbm{1}[\tau(f)\leq |f(t)|<2 \tau(f)]\bigg(\frac{f(t)}{2\tau(f)}\bigg)^{i} (2^2-1)\bigg]
    \\ & \quad  +\mathbb{E}_{t\sim \mathcal{N}(0,1)}\bigg[\mathbbm{1}[|f(t)|< \tau(f)]\bigg(\frac{\tau(f)}{2\tau(f)}\bigg)^{i}(0^2-1)\bigg] 
    \\ \geq &~ 3\mathbb{P}_{t\sim \mathcal{N}(0,1)}[|f(t)|\geq 2\tau(f)]-2^{-i}.
    \end{align}
    Note that $\mathbb{P}[|f(t)|\geq 2\tau(f)]$ is positive (since $f$ is polynomial) and independent of $i$, while $2^{-i}$ decays to $0$ as $i$ increases.
    Therefore, for sufficiently large $i\in \mathbb{N}$, \eqref{eq:PT-1} is positive and hence ${\mathrm{IE}}(f^{i})\le 2$. The subsequent analysis aims to provide an upper bound on $i$. 

    \noindent {\bf (i-2): Construction of test function.}
    We introduce the notation $H(\cdot ;j)$ which takes any function (in $L^1$) and returns its $j$-th Hermite coefficient. We consider the following test function:
    \begin{align}
    \mathscr{H}(f):=
       \sum_{i=2}^\infty \bigg(\frac{H(f^i;2)}{2^{\frac{i}{2}}(2i-1)^{\frac{iq}{2}}}\bigg)^2.
        \label{eq:PT-2}
    \end{align}

    \noindent {\bf (i-3): Lower bound of test function via compactness.}
    Let $\mathcal{F}_q$ be a set of polynomials with degree up to $q$ with unit $L^2$ norm. 
    Because $\mathscr{H}(f)$ is positive for any $f\in\mathcal{F}_q$, $H(f^i;2)$ is continuous with respect to $f$, and $\mathcal{F}_q$ is a compact set, $\inf_{f\in\mathcal{F}_q} \mathscr{H}(f)$ admits a minimum value $\mathscr{H}_0$ which is positive.

    \noindent {\bf (i-4): Conclusion via hypercontractivity.}
    Because $f$ is a polynomial with degree at most $q$, Gaussian hypercontractivity \cite{odonnell2014analysis} yields that
    \begin{align}
      2 H(f^i;2)^2 \leq \mathbb{E}_{t\sim \mathcal{N}(0,1)}\big[(f(t))^{2i}\big] \leq (2i-1)^{iq} \big(\mathbb{E}_{t\sim \mathcal{N}(0,1)}\big[f(t)^2\big]\big)^i
        =(2i-1)^{iq}.
        \label{eq:PT-3}
    \end{align}
    Therefore, for all polynomials in $\mathcal{F}_q$, a partial sum of \eqref{eq:PT-2}
    is uniformly bounded by
    \begin{align}
       \bigg|\sum_{i=j}^\infty \bigg(\frac{H(f^i;2)}{2^{\frac{i}{2}}(2i-1)^{\frac{iq}{2}}}\bigg)^2\bigg|\leq \sum_{i=j}^\infty 2^{-i-1}=2^{-j}\to 0\quad (j\to \infty)
      .
    \end{align}
    Combining this with the fact that $\mathscr{H}(f)\geq \mathscr{H}_0>0$, we know that there exists some $C_q\leq 1+\log_2 (\mathscr{H}_0^{-1})$ such that
    \begin{align}
       \sum_{i=2}^{C_q} \bigg(\frac{H(f^i;2)}{2^{\frac{i}{2}}(2i-1)^{\frac{iq}{2}}}\bigg)^2>\frac12 \mathscr{H}_0>0, 
    \end{align}
    for all polynomials in $\mathcal{F}_q$.
    This means that there is at least one $i\leq C_q $ such that $H(f^i;2)\ne 0$. 

\subsection{Proof for Non-even Functions $(ii)$}

    \noindent {\bf (ii-1): Monomials reducing the information exponent.} 
    We prove that some exponentiation of $g:=f^2$ has non-zero first Hermite coefficient. Denote $g^{\mathrm{odd}}$ as the odd part of $g$, and similarly $g^{\mathrm{even}}$. 
    Let $\upsilon(g)\in \mathbb{R}_+$ be the value at which the followings hold: 
    \begin{itemize}[leftmargin=7mm]
        \item[(a)] $g^{\mathrm{odd}}(t)>0$ for all $t\geq\upsilon(g) $ and  $g^{\mathrm{odd}}(t)<0$ for all $t\leq-\upsilon(g)$.
        \item[(b)] $g^{\mathrm{even}}(t)>|g^{\mathrm{odd}}(t)|$ for all $t\geq\upsilon(g) $ and $t\leq-\upsilon(g)$.
        \item[(c)] For for all $t\geq\upsilon(g) $ and
$t\leq-\upsilon(g)$, $g(s)=g(t)$ (as an equation of $s$) only has two real-valued solutions with opposing signs.
\end{itemize}
    Such threshold $\upsilon(g)$ exists because the tail of $g=f^2$ is dominated by the highest degree which is even. 
    Then, we let $\tau(g)=\max_{-\upsilon(g)\leq t\leq \upsilon(g)}|g(t)|$.
    
        Consider the following expectation:
    \begin{align}
        \mathbb{E}_{t\sim \mathcal{N}(0,1)}\bigg[\bigg(\frac{g(t)}{2\tau(g)}\bigg)^{i} t\bigg].
        \label{eq:PT-11}
    \end{align}
    \eqref{eq:PT-11} is decomposed as
    \begin{align}
      \eqref{eq:PT-11}
     =&~ \mathbb{E}_{t\sim \mathcal{N}(0,1)}\bigg[\mathbbm{1}[|g(t)|\geq 3\tau(f)]\bigg(\frac{g(t)}{3\tau(g)}\bigg)^{i} t\bigg] \\
        & \quad +\mathbb{E}_{t\sim \mathcal{N}(0,1)}\bigg[\mathbbm{1}[2\tau(g)\leq |g(t)|<3 \tau(g)]\bigg(\frac{g(t)}{3\tau(f)}\bigg)^{i} t\bigg]
      \\ & \quad  +\mathbb{E}_{t\sim \mathcal{N}(0,1)}\bigg[\mathbbm{1}[|g(t)|< 2\tau(g)]\bigg(\frac{g(t)}{3\tau(g)}\bigg)^{i} t\bigg]
       \label{eq:PT-12}.
    \end{align}
    We first evaluate the first term. 
    Because of (c), $g(t)= 3\tau(f)$ has two real-valued solutions $\alpha<0<\beta$.
    Because of (a) and (b), $g(\beta) =g^{\mathrm{even}}(\beta)+g^{\mathrm{odd}}(\beta)=3\tau(f)>g^{\mathrm{even}}(-\beta)+g^{\mathrm{odd}}(-\beta)=g^{\mathrm{odd}}(-\beta)$.
    Because $\lim_{t\to -\infty}g^{\mathrm{odd}}(t)=+\infty$, and $\alpha$ is the only solution in $t<0$, we have $\alpha<-\beta$.
    Moreover, for all $t>\beta$, we have $g(t) =g^{\mathrm{even}}(t)+g^{\mathrm{odd}}(t)>g^{\mathrm{even}}(-t)+g^{\mathrm{odd}}(-t)=g^{\mathrm{odd}}(-t)$.
    Combining the above, the first term of \eqref{eq:PT-12} is bounded as
    \begin{align}
        &\mathbb{E}_{t\sim \mathcal{N}(0,1)}\bigg[\mathbbm{1}[|g(t)|\geq 3\tau(f)]\bigg(\frac{g(t)}{3\tau(g)}\bigg)^{i} t\bigg]
        \\ & = \mathbb{E}_{t\sim \mathcal{N}(0,1)}\bigg[\mathbbm{1}[\beta\leq t\leq -\alpha]\bigg(\frac{g(t)}{3\tau(g)}\bigg)^{i} t\bigg]
        +
        \mathbb{E}_{t\sim \mathcal{N}(0,1)}\bigg[\mathbbm{1}[t\geq -\alpha] \bigg(\frac{g(t)}{3\tau(g)}\bigg)^{i}t\bigg]
    \\ & \quad  +  \mathbb{E}_{t\sim \mathcal{N}(0,1)}\bigg[\mathbbm{1}[t\leq \alpha] \bigg(\frac{g(t)}{3\tau(g)}\bigg)^{i}t\bigg]
      \\ & =  \mathbb{E}_{t\sim \mathcal{N}(0,1)}\big[\mathbbm{1}[\beta\leq t\leq -\alpha] t\big]
        +\mathbb{E}_{t\sim \mathcal{N}(0,1)}\bigg[\mathbbm{1}[t\geq -\alpha]\bigg(\bigg(\frac{g(t)}{3\tau(g)}\bigg)^{i}- \bigg(\frac{g(-t)}{3\tau(g)}\bigg)^{i}\bigg)t\bigg]
       \\ &  >\beta\mathbb{P}_{t\sim \mathcal{N}(0,1)}\big[\beta\leq t\leq -\alpha\big].
    \end{align} 
   Following the exact same reasoning, we know that the second term of \eqref{eq:PT-12} is positive.
Finally, the third term which is bounded by
    \begin{align}
        \mathbb{E}_{t\sim \mathcal{N}(0,1)}\big[\mathbbm{1}[|g(t)|< 2\tau(g)]\bigg(\frac{g(t)}{3\tau(g)}\bigg)^{i} t\big]
        \geq 
       - \mathbb{E}_{t\sim \mathcal{N}(0,1)}\big[\mathbbm{1}[|g(t)|< 2\tau(g)]|t|\big]\bigg(\frac{2}{3}\bigg)^{i}.
    \end{align}

    Putting things together, 
    \begin{align}
        \eqref{eq:PT-12} > \beta\mathbb{P}_{t\sim \mathcal{N}(0,1)}\big[\beta\leq t\leq -\alpha\big]
        -\mathbb{E}_{t\sim \mathcal{N}(0,1)}\big[\mathbbm{1}[|g(t)|< 2\tau(g)]|t|\big]\bigg(\frac{2}{3}\bigg)^{i}.
    \end{align}
    The first term is independent of $i$ and positive, while the second term goes to zero as $i$ grows.
    Therefore, there exists some $i$ such that $\mathrm{IE}(g^i;1)=1$.

    \noindent {\bf (ii-2): Construction of test function.}
    This time we consider the following function:
    \begin{align}
    \mathscr{H}(f):=
       \sum_{i=2}^\infty \bigg(\frac{H(f^i;1)}{2^{\frac{i}{2}}(2i-1)^{\frac{iq}{2}}}\bigg)^2
       .
        \label{eq:PT-13}
    \end{align}

    \noindent {\bf (ii-3): Lower bound of test function via compactness.}
    Let $\mathcal{F}_q$ be a set of unit $L^2$-norm polynomials with degree up to $q$ and $\mathbb{E}_{t\sim \mathcal{N}(0,1)}[f^{\mathrm{odd}}(t)^2]\geq c$. 
    Since $ \mathscr{H}(f)$ is always positive for $\mathcal{F}_q$, $ \mathscr{H}(f)$ is continuous with respect to $f$, and $\mathcal{F}_q$ is a compact set, $\inf_{f\in\mathcal{F}_q}  \mathscr{H}(f)$ has the minimum value $ \mathscr{H}_0$ that is positive.
    Note that $\mathscr{H}(f)$ might depends on $c$.

    \noindent {\bf (ii-4): Conclusion via hypercontractivity.}
    Using the same argument as in (i),  we conclude that there exists some $C_{q,c}$ such that 
    \begin{align}
       \sum_{i=2}^{C_q} \bigg(\frac{H(f^i;1)}{2^i(2i-1)^{\frac{iq}{2}}}\bigg)^2>\frac12 \mathscr{H}_0>0. 
    \end{align}
    Because $\mathscr{H}_0$ depends on $c$, $C_{q,c}$ depends on $c$ as well as $q$.
\end{proofof}

\subsection{Proof for Non-Polynomial Functions}\label{subsection:Reduction_Generative}
For non-polynomial link functions, we note that similar to \cite{damian2024computational}, the existence of polynomial basis is needed to exclude extreme cases, and we cannot upper bound the required degree $I$ because general link functions are not included in a compact space. 

\begin{proofof}[Lemma~\ref{lemm:Reduction_Generative}]
The derivation is analogous to \cite[Lemma F.14]{damian2024computational}. Let $z\sim \mathcal{N}(0,1)$ and $y=\sigma_*(z)$. 
We define $\zeta_{p_*}(y)=\mathbb{E}[\frac{1}{\sqrt{p_*!}}\He_{p_*}(z)|y]$ and its basis expansion $\zeta_{p_*}(y)=\sum_{k=0}^\infty v_k \phi_k$.
Let $K$ be a smallest integer such that $v_k\ne 0$.
Then, there exists an integer with $I\leq K$ such that $\mathrm{IE}(y^I)=p_*$. 
Indeed, 
\begin{align}
    \mathbb{E}[\phi_K(y)\He_{p_*}(z)]
    &=\mathbb{E}_y[\Phi_K(y)\mathbb{E}_{z|y}[\He_{p_*}(z)|y]]
    \\ & =\mathbb{E}_y\bigg[\Phi_K(y)\sum_{k=0}^K v_k \phi_k (y)\bigg]
    =v_K\ne 0,
\end{align}
which means that at least one of $y,y^2,\cdots,y^K$ yields a non-zero $p_*$-th Hermite coefficient. 
\end{proofof}

\section{SGD with Reused Batch}
\label{app:SGD}

In this section we show that Algorithm~\ref{alg:main} learns single-index models in $\tilde{O}(d^{1\lor(p_*-1)})$ samples with high probability.
The algorithm trains the first layer for $T_1$ SGD steps, where we sample a new data point in every two steps. 
The first layer training is further divided into two phases: weak recovery ($\vw^\top \vtheta \gtrsim 1$) and strong recovery ($\|\vw  - \vtheta\| \lesssim \varepsilon$). Then, we learn the second layer parameters.

Specifically, Section~\ref{subsection:Initialization} shows that at initialization, a (nearly) constant fraction of neurons has alignment $\vw^\top \vtheta$ beyond a certain threshold.
We focus on such neurons in the first phase of training. 
Section~\ref{subsection:Expected} lower bounds the expected update of alignment $\vw^\top \vtheta$ of two gradient steps, and Section~\ref{subsection:Stochastic} establishes that the neurons achieve weak recovery within $2T_{1,1}=\tilde{O}(d^{1\lor (p_*-1)})$ steps.
Section~\ref{subsection:Final} discusses how to convert weak recovery to strong recovery using $2T_{1,2}=\tilde{O}(d\varepsilon^{-2})$ more steps.
We let $T_1=2T_{1,1}+2T_{1,2}$.
Finally, Section~\ref{subsection:Second} analyzes second layer training and concludes the proof. 

In the following proofs, we use several constants, which depends on $d$ at most at most polylogarithmically.
Specifically, asymptotic strength of the constants is ordered as follows.
\begin{align}
   1 \simeq c_\sigma \simeq C_1 \lesssim 
        \left\{ \begin{matrix}
        c_\eta^{-1} \simeq C_2 \lesssim \mathrm{poly}(c_\eta^{-1})\lesssim\left\{\begin{matrix} c_1^{-1}\simeq C_3\\ c_2^{-1} \\\end{matrix}\right\}\\
        \delta^{-1} \\
        \end{matrix} \right\}
   & \lesssim \left\{ \begin{matrix}\delta^{-1}\mathrm{poly}(c_\eta^{-1})\lesssim c_\xi^{-1} \\ \mathrm{poly}(c_1^{-1})\lesssim \bar{c}_\eta^{-1}\end{matrix} \right\}
   \\ & \lesssim\mathrm{polylog}(d) =C_4.
\end{align}
% $r_\beta$ in the main text can be taken as $r_\beta = \cE$. 
Here, $c_\eta$ and $\delta$ should satisfy $\lim_{d\to \infty}c_\eta =\lim_{d\to \infty}\delta =0$, but the convergence can be arbitrarily slow, (e.g., as slow as $1/\log\log\log \cdots \log d$).
This requirement comes from the fact that we do not know the exact value of $H(\sigma_*^I;p_*)$.
To ensure that one signal term (from the Taylor series) is isolated, taking $\eta\asymp d^{-1}$ with a sufficiently small constant is insufficient but $\eta\asymp c_\eta d^{-1}$ with arbitrarily slow $c_\eta$ suffices. 
Also, to guarantee that the failure probability is $o_d(1)$, we require $\delta$ to be $o_d(1)$.
$c_\xi$ can also decay arbitrarily slowly, as long as it satisfies $c_\xi\lesssim \delta\mathrm{poly}(c_\eta^{-1})$. 
$C_4=\mathrm{polylog}(d)$ will be used to represent any polylogarithmic factor that comes from high probability bounds.

For the first-layer training, we can reduce the argument into training of one neuron using the correlation loss as follows. 
At each step, the gradient update (Line 8 of Algorithm 1) is written as
\begin{align}\vw^{t+1}_j &\leftarrow \vw^t_j -\eta^t
   \tilde{\nabla}_{\vw}\big( (f_{\vTheta}(\vx) -y)^2\big) 
 \\ &  = \vw^t_j - \eta^t \tilde{\nabla}_{\vw}\bigg(\frac1N \sum_{j=1}^N a_j \sigma_j({\vw_j^t}^\top \vx)\bigg)^2 
 +2 \eta^t_j \tilde{\nabla}_{\vw}\bigg(y \frac1N \sum_{j=1}^N a_j \sigma_j({\vw_j^t}^\top \vx)\bigg)
 \\ &  =  \vw^t_j - \frac{2\eta^t c_a^2}{N} \bigg(\frac1N \sum_{j=1}^N  \sigma_j({\vw_j^t}^\top \vx)\bigg) \big(\tilde{\nabla}_{\vw} \sigma_j({\vw_j^t}^\top \vx)\big) 
 +\frac{2\eta^t c_a}{N} y\big(\tilde{\nabla}_{\vw}\sigma_j({\vw_j^t}^\top \vx)\big)
 .
 \label{eq:IgnoreInteraction}
\end{align}
While the second term scales with $ \eta^tc_a^2 N^{-1}$, the third term scales with $\eta^tc_a N^{-1}$.
Thus, by setting $c_a$ sufficiently small, we can ignore the interaction between neurons.
We will show that the strength of the signal 
in the direction of $\vtheta$ is at least $(\kappa_j^{t})^{p_*-1} \gtrsim d^{-\frac{p_*-1}{2}}$ (up to a polylogarithmic factor, and $p_*=\mathrm{GE}(\sigma_*)$).
On the other hand, we can easily see that $\vtheta^\top \big(\frac1N \sum_{j=1}^N  \sigma_j({\vw_j^t}^\top \vx)\big) \big(\tilde{\nabla}_{\vw} \sigma_j({\vw_j^t}^\top \vx)\big) $ is bounded by $\tilde{O}(1)$ with high probability.
Therefore, by simply letting $c_a= \tilde{\Theta}(d^{-\frac{p_*-1}{2}})$, we can ignore the effect of the second term in \eqref{eq:IgnoreInteraction}. 
Moreover, for simplicity, we will reparameterize $\frac{2\eta^t c_a}{N}$ as $\eta^t$ below. 
Consequently, we may analyze the following update
\begin{align}\vw^{t+1}_j \leftarrow \vw^t_j + \eta^t \tilde{\nabla}_{\vw}\big(y  \sigma_j({\vw^t_j}^\top \vx)\big)
,\label{eq:OneNeuronCorr}
\end{align}
instead of Line 8 of Algorithm 1.
Since there is no interaction between neurons now, we omit the subscript $j$ when the context is clear.

\subsection{Assumptions on Link Function}\label{subsection:linkfunction}

The analysis consists of three different phases: weak recovery and strong recovery of the first-layer weights, and approximation of the link function (ridge regression of the second-layer).
Each phase requires different assumptions on the activation functions, depending on the link function.
Before starting the analysis, we decompose Assumptions~\ref{assump:student1} and \ref{assump:student2} and clarify which conditions are needed in each phase.
We prove that instead of using a specific activation function tailored to different link functions, a randomized activation function satisfies all required assumptions with probability $\Omega(1)$.

In the following, we write the student activation function as
\begin{align}
  \sigma_j(s):= \sum_{i=0}^{\infty}\beta_{j,i} \He_i(s)
\end{align}
with coefficients $\{\beta_{j,i}\}_{i=0}^{C_\sigma}$ (sometimes the subscript $j$, which is the index of the neurons, is omitted).

\subsubsection{For polynomial link functions}

In the following, we summarize the precise conditions to be satisfied by the activation functions (these conditions are weaker than Assumptions~\ref{assump:student1}~and~\ref{assump:student2}). 
For polynomial link functions, we focus on polynomial activation functions (with bounded degree) for simplicity, but non-polynomial activation functions would not change the proof significantly.

Let $p$ and $q$ be the minimum and maximum degree of non-zero Hermite coefficients of $\sigma_*$.
Note that $\mathrm{GE}(\sigma_*)=1$ or $2$ holds (see Proposition~\ref{prop:monomial-transformation}).
Let $I\leq C_q$ (according to Proposition~\ref{prop:monomial-transformation}) be the smallest integer such that $\mathrm{IE}(\sigma_*^I)=\mathrm{GE}(\sigma_*)=p_*$ and $C_\sigma$ be the degree of the activation function.
\begin{itemize}[leftmargin=7mm]
    \item[{\bf (I)}] {\bf If $I=1 \Leftrightarrow \mathrm{IE}(\sigma_*)=\mathrm{GE}(\sigma_*)=p_*$.}
    \begin{itemize}
        \item[]{\bf Weak recovery:} $\alpha_{p_*}\beta_{p_*}>0$ (covered by Assumption~\ref{assump:student2}).
        \item[]{\bf Strong recovery:} $\sum_{j=p_*}^{q} j!\alpha_j\beta_j s^{j-1}>0$ for all $s>0$ (covered by Assumption~\ref{assump:student1}).
        \item[]{\bf Approximation (ridge regression):} $\beta_i \ne 0$ for some $i\geq q$  (covered by Assumption~\ref{assump:student2}).
    \end{itemize}
    \item[{\bf (II)}] {\bf If $2\leq I=\{\min\ i \mid \mathrm{IE}(\sigma_*^I)=\mathrm{GE}(\sigma_*)=p_*\}\leq C_\sigma$.}
    \begin{itemize}
        \item[]{\bf Weak recovery:} $H((\sigma_*)^I;p_*)H(\sigma^{(I)}(\sigma^{(1)})^{I-1};p_*-1) >0 $  (covered by Assumption~\ref{assump:student2}).
        \item[]{\bf Strong recovery:}  $\sum_{j=p_*+1}^{q} j!\alpha_j\beta_j s^{j-1}>0$ for all $s>0$  (covered by Assumption~\ref{assump:student1}).
        \item[]{\bf Approximation:} $\beta_i \ne 0$ for some $i\geq q$  (covered by Assumption~\ref{assump:student2}).
    \end{itemize}
\end{itemize}
Note that it is difficult to construct a deterministic activation function that satisfies all of the assumptions for any link function $\sigma_*$ (the simplest counterexample is to consider $-\sigma_*$ which flips the Hermite coefficients).
Instead, we show the existence of randomized construction of such an activation function that satisfies all of the assumptions on the activation function simultaneously with constant probability, which entails that a subset of neurons can achieve strong recovery. 
The construction does not depend on properties of the link function itself except for its degree $q$. 
\begin{lemma}\label{lemma:DesignActivation-1}
There exists a randomized activation function sampled from a discrete set such that the above conditions hold with constant probability. 
\end{lemma}
\begin{proof}
    Let $c$ be a sufficiently small constant only used in this proof and $C_\sigma$ 
     be the minimum odd integer with $C_\sigma \geq \max\{C_q+1,q+2,3\}$, where $C_q$ was introduced in Proposition~\ref{prop:monomial-transformation}.
    With probability $\frac12$, we let  $\beta_1\sim \mathrm{Unif}(\{\pm 1\})$, 
    and $\beta_j\sim \mathrm{Unif}(\{\pm c\})$ 
    for $2\leq j\leq C_\sigma$.
     With probability $\frac12$, we let $\beta_j\sim \mathrm{Unif}(\{\pm c\})$
    for $1\leq j\leq C_\sigma-2$ and $\beta_{C_\sigma-1}=\beta_{C_\sigma}\sim \mathrm{Unif}(\{\pm 1\})$.

    We first consider (I). 
    When $\beta_1\sim \mathrm{Unif}(\{\pm 1\})$, 
    and $\beta_j\sim \mathrm{Unif}(\{\pm c\})$ 
    for $2\leq j\leq C_\sigma$, it is easy to see $\mathrm{sign}(\alpha_j)=\mathrm{sign}(\beta_j)$ for all $j=1,\cdots,q$ hold with probability at least $2^{-q}$, which is sufficient to satisfy (I).

    We then consider (II).
    First focus on the case when $p_*=1$ and $I$ is even.  
    When $\beta_1\sim \mathrm{Unif}(\{\pm 1\})$ 
    and $\beta_j\sim \mathrm{Unif}(\{\pm c\})$ for $2\leq j\leq C_\sigma$, 
    by taking $c$ sufficiently small, we have
    \begin{align}
        H(\sigma^{(I)}(\sigma^{(1)})^{I-1};0)
        = \underbrace{I!\beta_I (\beta_1)^{I-1}}_{\text{$\asymp c$}} +  O(c^2).\label{eq:DesignAct-6}
    \end{align}
    When $I$ is even, by adjusting the sign of $\beta_1$, $H(\sigma^{(I)}(\sigma^{(1)})^{I-1};0) $ is non-zero and has the same sign as $H((\sigma_*)^I;1)$ with probability $\frac12$.
    Note that the sign of $\beta_1$ is independent from whether $\sum_{j=2}^{q} j!\alpha_j\beta_j s^{j-1}>0$ for all $s>0$ holds.
    This holds with probability at least $2^{-q+1}$.  Thus we verified (II) for $p_*=1$ and even $I$.

    For $p_*=1$ and odd $I$, consider $\beta_j\sim \mathrm{Unif}(\{\pm c\})$
    for $1\leq j\leq C_\sigma-2$ and $\beta_{C_\sigma-1}=\beta_{C_\sigma}\sim \mathrm{Unif}(\{\pm 1\})$.
    Note that $\sum_{j=2}^{q} j!\alpha_j\beta_j s^{j-1}>0$ for all $s>0$ (this is the condition for strong recovery) and the condition for ridge regression also holds.
    Furthermore, the sign of $H((\He_{C_\sigma}+\He_{C_\sigma-1})^{(I)}((\He_{C_\sigma}+\He_{C_\sigma-1})^{(1)})^{I-1};0)$ is $\pm 1$ with equiprobability, independent of $\beta_2,\dots,\beta_q$.
    Therefore, by taking $c$ sufficiently small, we can obtain the desired sign of $H(\sigma^{(I)}(\sigma^{(1)})^{I-1};0)$.
    Thus we proved (II) for $p_*=1$ and odd $I$. 

    Regarding (II) for $p_*=2$ and even $I$, 
    when $\beta_1\sim \mathrm{Unif}(\{\pm 1\})$ 
    and $\beta_j\sim \mathrm{Unif}(\{\pm c\})$ for $2\leq j\leq C_\sigma$, we have
    \begin{align}
        H(\sigma^{(I)}(\sigma^{(1)})^{I-1};1)
       =\underbrace{(I+1)!\beta_{I+1} (\beta_1)^{I-1}}_{\text{$\asymp c$}} +  O(c^2).
    \end{align}
    Thus, similar to (II) with $p_*=1$ and even $I$, we get (II) for $p_*=2$ and even $I$. 

    Finally, consider (II) for $p_*=2$ and odd $I$.
    When $\beta_j\sim \mathrm{Unif}(\{\pm c\})$
    for $1\leq j\leq C_\sigma-2$ and $\beta_{C_\sigma-1}=\beta_{C_\sigma}\sim \mathrm{Unif}(\{\pm 1\})$, the sign of $H((\He_{C_\sigma}+\He_{C_\sigma-1})^{(I)}((\He_{C_\sigma}+\He_{C_\sigma-1})^{(1)})^{I-1};1)$ is $\pm 1$ with equiprobability when $I$ is odd, and this term dominates the others in $H(\sigma^{(I)}(\sigma^{(1)})^{I-1};1)$. 
    Thus, (II) for $p_*=2$ and odd $I$ holds similarly to (II) for $p_*=1$ and odd $I$.
    
    Now we have obtained the assertion for all cases.
\end{proof}

\subsubsection{For general link functions}\label{subsubsection:Assumption-GeneralLink}

Now we consider non-polynomial link functions with potentially large generative exponent $p_*=\mathrm{GE}(\sigma_*)\ge 2$.
For weak and strong recovery to succeed, the conditions on the activation function are essentially the same as those for polynomial link functions: 
\begin{itemize}[leftmargin=7mm]
    \item[{\bf (I)}] {\bf If $I=1 \Leftrightarrow \mathrm{IE}(\sigma_*)=\mathrm{GE}(\sigma_*)=p_*$.}
    \begin{itemize}
        \item[]{\bf Weak recovery:} $\alpha_{p_*}\beta_{p_*}>0$.
        \item[]{\bf Strong recovery:} $\sum_{j=p_*}^{\infty} j!\alpha_j\beta_j s^{j-1}>0$ for all $s>0$,
    \end{itemize}
    \item[{\bf (II)}] {\bf If $2\leq I=\{\min\ i \mid \mathrm{IE}(\sigma_*^I)=\mathrm{GE}(\sigma_*)=p_*\}\leq C_\sigma$.}
    \begin{itemize}
        \item[]{\bf Weak recovery:} $H((\sigma_*)^I;p_*)H(\sigma^{(I)}(\sigma^{(1)})^{I-1};p_*-1) >0 $,
        \item[]{\bf Strong recovery:}  $\sum_{j=p_*+1}^{\infty} j!\alpha_j\beta_j s^{j-1}>0$ for all $s>0$.
    \end{itemize}
\end{itemize}
Due to the proof strategy (which uses Taylor expansion), we also require that all differentials and sum of expectations appearing in the following proofs are well-defined and bounded. 

To approximate a non-polynomial $\sigma_*$, we introduce the following condition on the activation function. 
We sample $\sigma_j$ from a discrete set (with bounded cardinality).
Let $J$ be an index set such that the coefficients of $\sigma_j\ (j\in J)$ satisfy the conditions above.
Because we are selecting $\sigma_j$ from a discrete set, $|J| \simeq N$ holds. 
We introduce the following condition, which states that the target single-index model can be well-approximated by a linear combination of student neurons. % (which automatically holds for the polynomial case).
\begin{assumption}\label{assump:Approx-General}
    When $b_j\sim \mathrm{Unif}([-C_b,C_b])$ where  $(C_b = \mathrm{polylog}(d))$ and $\vx_1,\dots,\vx_{T_2}\sim \mathcal{N}(0,\vI_d)$, there exists a set of coefficients $a_1,\dots,a_{|J|}$ such that
    \begin{align}
      \frac{1}{T_2} \sum_{i=1}^{T_2} \bigg(\frac{1}{|J|} \sum_{j\in J}a_j \sigma_j(\vtheta_j^\top \vx_i + b_j) - \sigma_*(\vtheta^\top \vx_i)\bigg)^2 \lesssim \varepsilon^2,
    \end{align}
    holds with coefficients of reasonable magnitudes $\sum_{j\in J}a_j^2 = \Theta(|J|)$ with high probability (w.r.t. the randomness of $b_j$ and $\vx_i$).  
    Moreover, $\mathbb{E}_{\vx}[\sigma_j(\vtheta_j^\top \vx+ b_j)^4] \leq \mathrm{polylog}(d)$ for all $j$ with high probability (w.r.t. the randomness of $b_j$).  
\end{assumption}
The following lemma states that we can design a randomized activation function that satisfies all of the above assumptions with probability $\Omega(1)$, as long as the link function $\sigma$ satisfies Assumption~\ref{assump:Approx-General} for $\sigma = \mathrm{ReLU}$.
In other words, we are able to cover the class of link functions $\sigma$ that can be efficiently approximated by a two-layer ReLU network. 
Since the general link functions are not included in a compact space, we do not have an upper bound of exponent to obtain $\mathrm{IE}(\sigma_*^I)=\mathrm{GE}(\sigma_*)$ as we had $C_q$ in the polynomial case. 
Consequently, our student activation is not entirely agnostic to the link function $\sigma_*$, as we require knowledge of $p$ (information exponent), $p_*$ and $I$. 
\begin{lemma}\label{lemma:DesignActivation-2}
Suppose the target link function $\sigma_*$ satisfies Assumption~\ref{assump:Approx-General} for $\sigma_j = \mathrm{ReLU}$. 
There exists a randomized activation sampled from a discrete set such that the above conditions hold with constant probability. 
\end{lemma}
Before we sketch the design of activation function, we present the following approximation result from \cite{bietti2022learning}, which establishes that Assumption~\ref{assump:Approx-General} with $\sigma_j = \mathrm{ReLU}$ is satisfied for broad class of functions, according to Lemma 4.4 and 4.5 of \cite{bietti2022learning}. 
Specifically, taking $\tau = 1/2$ and $\lambda = N^{-1}$ yields that $\mathbb{E}_{\vx}[(\frac{1}{|J|}\sum_{j\in J}a_j\sigma_j (\vtheta^\top \vx + b_j) - \sigma_*(\vtheta^\top \vx))^2]\leq N^{-\frac27}$.  
Although they sample $b_j$ from Gaussian $\mathcal{N}(0,2)$, the result translates to uniform sampling of biases from $[-C_b,C_b]$ by introducing additional logarithmic factor.
\begin{lemma}[Lemma 4.4, 4.5 of \cite{bietti2022learning}]\label{lemma:ReLU_settles_everything}
    Suppose that $\E_{z\sim\cN(0,1)}[\sigma_*(z)^4]$, $\E_{z\sim\cN(0,1)}[\sigma_*''(z)^4] < \infty$. 
    Then, Assumption~4 with $\sigma_j = \mathrm{ReLU}$ holds with $\varepsilon = N^{-\frac17}$ and $C_b \simeq \sqrt{\log d}$. 
\end{lemma}
\begin{proofof}[Lemma~\ref{lemma:DesignActivation-2}]
    We show the existence of suitable $\sigma$ in two steps: first we construct a randomized polynomial activation function that satisfies conditions $\text{(I)(II)}$ with constant probability; then we add a small ReLU perturbation so that the activation can approximate non-polynomial $\sigma_*$. 
    
    Recall $p\in\N_+$ is the information exponent of $\sigma_*$. 
    We first show that there exists a randomized polynomial activation that satisfies the conditions for weak and strong recovery with probability $\Omega(1)$.
    Note that the issue of differentiability and bounded moment is avoided when we focus on the polynomial activation functions.
    We specify the following two distributions. With probability $\frac12$, let $\beta_1 \sim \mathrm{Unif}(\{-1,1\})$,  $\beta_j \sim \mathrm{Unif}(\{-c,c\})$ for $j=1,\cdots,p_*+I-1$ and $\beta_j=0$ otherwise, where $c>0$ is a sufficiently small constant.
    With probability $\frac12$, let $\beta_1 = \mathrm{Unif}(\{-1,1\})$, $\beta_2 = \mathrm{Unif}(\{-c,c\})$, 
    $\beta_j = \mathrm{Unif}(\{-c^2,c^2\})$ for all $2\leq j \leq (p_*+I)\lor p$ for a sufficiently small constant $c>0$, and $\beta_j = 0$ otherwise.

    Regarding (I), consider the case when the coefficients are sampled from the first distribution, and $|\beta_j| \ll 1$ except for $j=p_*$.
    Then, $\sum_{j=p_*}^{\infty} j!\alpha_j\beta_j s^{j-1}\approx p_*!\alpha_{p_*}\beta_{p_*}s^{p_*}$.
    Choosing the sign of $\beta_{p_*}$, we have that the assumption holds with probability $\Omega(1)$.

    Regarding (II) with even $I$, 
    consider coefficients sampled from the first distribution, and $\mathrm{Sign}(\beta_j) = \mathrm{Sign}(\alpha_j)$ for $j\leq (p_*+I-1)\lor p$.
    Then, $\sum_{j=p_*+1}^{\infty} j!\alpha_j\beta_j s^{j-1}>0$ for all $s>0$.
    Also, similarly to \eqref{eq:DesignAct-6}, 
    \begin{align}
        H(\sigma^{(I)}(\sigma^{(1)})^{I-1};p_*-1)
        = \underbrace{i!\beta_{I+p_*-1} (\beta_1)^{I-1}}_{\text{$\asymp c$}} +  O(c^2).\label{eq:DesignAct-7}
    \end{align}
    By flipping the sign of $\beta_1$, we can change the sign of $H(\sigma^{(I)}(\sigma^{(1)})^{I-1};p_*-1)$.
    Thus, (II) for even $I$ is satisfied by a randomized choice of $\beta_1$.
    
    For (II) with odd $I$, 
    consider coefficients sampled from the second distribution, and $\mathrm{Sign}(\beta_j) = \mathrm{Sign}(\alpha_j)$ for $j\leq (p_*+I)\lor p$.
    Then, $\sum_{j=p_*+1}^{\infty} j!\alpha_j\beta_j s^{j-1}>0$ for all $s>0$.
    \begin{align}
        &H(\sigma^{(I)}(\sigma^{(1)})^{I-1};p_*-1)
        =\frac{1}{(p_*-1)!}\mathbb{E}[\sigma^{(I)}(\sigma^{(1)})^{I-1}\He_{p_*-1}]
      \\ & = \frac{1}{(p_*-1)!}\mathbb{E}[(I-1)(\beta_{p_*+I}\He_{p_*+I})^{(I)}(\beta_{2}\He_{2})^{(1)}(\beta_1)^{I-2}\He_{p_*-1}] + O(c^4).
        \\ & = \underbrace{\frac{2(I-1)\beta_{p_*+I}\beta_2(\beta_1)^{I-2}(p_*+I)!}{(p_*-1)!}}_{\text{$\asymp c^3$}} +  O(c^4).\label{eq:DesignAct-8}
    \end{align}
     By flipping the sign of $\beta_1$, we can change the sign of $H(\sigma^{(I)}(\sigma^{(1)})^{I-1};p_*-1)$.
     Thus, (II) for odd $I$ is satisfied by a randomized choice of $\beta_1$.

     Therefore, we have constructed a randomized polynomial activation $\sigma$ that satisfies all of the conditions for weak and strong recovery. 
     Now we provide a sketch of reasoning that when the link function $\sigma_*$ is well-approximated by ReLU as specified in Assumption~\ref{assump:Approx-General}, we can find some $\sigma$ that additionally satisfies Assumption~\ref{assump:Approx-General} by introducing a small ReLU component. Specifically, we add $c_{\mathrm{R}}\cdot\mathrm{ReLU}$ to the activation function with probability $\frac12$, with a sufficiently small $c_{\mathrm{R}}=\tilde{\Omega}(1)$, e.g., $c_{\mathrm{R}} = (\text{log}d)^{-C}$ for some $C>0$. 
     When a two-layer ReLU network approximates $\sigma_*$ that satisfies Assumption \ref{assump:Approx-General}, by using the neurons with added ReLU component, $\sigma_*$ can be approximated up to some polynomial residual with degree $(p_*+I)\lor p$.
     And by using the remaining polynomial neurons, we can approximate the additional polynomial terms in $\sigma_*$ (see Lemma~\ref{lemma:DamianApproximationPolynomial},\ref{lemma:ApproxPolynomialByHeq}). Subtracting the latter from the former, we obtain the desired approximation result.
     When $c_{\mathrm{R}}$ is sufficiently small, this additional term does not impact the conditions for weak and strong recovery and the moment calculations; similarly, since $c_{\mathrm{R}}\ll 1$ we may discard this non-smooth term before Taylor expansion without affecting the analysis of optimization dynamics. 
     We remark that to avoid such unnatural design of activation function, we can also train the first-layer parameters using a polynomial activation specified above, and then perturb it before the second-layer training to enhance the approximation ability --- such strategy has also been employed in prior layer-wise training analysis \cite{abbe2022merged}. 
\end{proofof}

\subsubsection{More Discussion on Assumption~\ref{assump:student1}}
\label{subsubsection:Moredicsussionon}
Assumption~\ref{assump:student1} requires $H(\sigma^{(I)}(\sigma^{(1)})^{I-1};p_*-1)$ is not zero and has the same sign as $H(\sigma_*^I;p_*)$. 
We remark that if we allow a negative momentum parameter larger than $1$, i.e., setting $\xi^{2(t+1)}=1+c_\xi d^{-\frac{(p_*-2)_+}{2}}$, we can negate the opposite sign of $H(\sigma^{(I)}(\sigma^{(1)})^{I-1};p_*-1)$ (see Lemma~\ref{lemma:PopulationPoly}), and the subsequent analysis still holds.
Therefore, what we essentially need is $H(\sigma^{(i)}(\sigma^{(1)})^{i-1};k) \ne 0$.
Lemma~\ref{lemm:non-zero-assumption} confirms that it is satisfied by almost all polynomials:
\begin{proofof}[Lemma~\ref{lemm:non-zero-assumption}]
    We note that $H(\sigma^{(i)}(\sigma^{(1)})^{i-1};k)=\mathbb{E}[\sigma^{(i)}(\sigma^{(1)})^{i-1}\He_k]$ is a polynomial of $\{\beta_j\}_{j=0}^{C_\sigma}$.
    This polynomial is not identically equal to zero.
    To confirm this, consider $\sigma = x^{C_\sigma}+x^{C_\sigma-1}$.
    Because $\sigma^{(i)}(\sigma^{(1)})^{i-1}$ is expanded as a sum of $x^l (i(C_\sigma-3)\leq l \leq i(C_\sigma-2)+1$ with positive coefficients and each $x^l$ is a sum of $\He_{l},\He_{l-2}\cdots$ with positive coefficients, $\sigma^{(i)}(\sigma^{(1)})^{i-1}$ has all positive Hermite coefficients for degree $0,1,\cdots,i(C_\sigma-2)+1$.
    If $k\leq i(C_\sigma-2)+1$, this choice of $\sigma$ yields $H(\sigma^{(i)}(\sigma^{(1)})^{i-1};k)>0$, which confirms that $H(\sigma^{(i)}(\sigma^{(1)})^{i-1};k)$ as a polynomial of $\{\beta_j\}_{j=0}^{C_\sigma}$ is not identically equal to zero. 
    Hence the assertion follows from so-called Schwartz–Zippel Lemma \cite{schwartz1980fast}, or the fact that zeros of a non-zero polynomial form a measure-zero set.
\end{proofof}

\subsection{Initialization}\label{subsection:Initialization}
We first consider the initial alignment. 
In the following sections, we focus on the neurons that satisfy $\kappa^0_j = \vtheta^\top \vw^0_j \geq 2c_\eta^{-1} d^{-\frac12}$ at the initialization.
The following lemma states that roughly a constant portion of the neurons satisfy the initial alignment condition upon random initialization.
In particular, if we take $c_\eta = \Omega((\log\log d)^{-\frac12})$, 
the fraction of neurons that satisfy the initial alignment condition is at least $ e^{-16c_\eta^{-2}} =\tilde{\Omega}(1)$.
Let us write $C_2 = c_\eta^{-1}$ for simplicity in the following.

\begin{lemma}\label{lemma:Initialization}
    At the time of initialization, $\kappa^0_j=\vtheta^\top \vw^0$ satisfies the following: 
    \begin{align}
        \mathbb{P}[\kappa^0_j \geq 2\CC d^{-\frac12}] =\mathbb{P}[\kappa^0_j \leq - 2\CC d^{-\frac12}]\gtrsim  e^{-16\CC^2} =\tilde{\Omega}(1).
    \end{align}
\end{lemma}
We make use of the following lemma.
\begin{lemma}[Theorem 2 of \citep{chang2011chernoff}]\label{lemma:chang2011chernoff}
    For any $\beta>1$ and $s\in \mathbb{R}$, we have
    \begin{align}
        \frac{\sqrt{2e(\beta-1)}}{2\beta\sqrt{\pi}}
        e^{-\frac{\beta s^2}{2}}
    \leq
        \int_{s}^\infty \frac{1}{\sqrt{2\pi}}e^{-\frac{t^2}{2}}\mathrm{d}t
    \end{align}
\end{lemma}
\begin{proofof}[Lemma~\ref{lemma:Initialization}]
    Because $\kappa^0 = \vv^\top \vw  \overset{\mathrm{d}}{=}\frac{{\ve}_1^\top \vg}{\|\vg\|}$, where $\vg \sim \mathcal{N}(0,\vI_d)$, 
    \begin{align}
        \mathbb{P}[\kappa^0_j \geq 2\CC d^{-\frac12}]
      &  = \mathbb{P}_{\vg\sim \mathcal{N}(0, \vI_d)}\bigg[{\ve}_1^\top \vg \geq 4\CC \land \|\vg\|\leq 2d^{\frac12}\bigg]
        \\ &
        \geq \mathbb{P}_{\vg\sim \mathcal{N}(0, \vI_d)}\bigg[{\ve}_1^\top \vg \geq 4\CC\bigg]
        -\mathbb{P}_{\vg\sim \mathcal{N}(0, \vI_d)}\bigg[\|\vg\|\geq 2d^{\frac12}\bigg]
        \\ & \gtrsim 
        \frac{\sqrt{2e(\beta-1)}}{2\beta\sqrt{\pi}}
        e^{-8\beta \CC^2}-e^{-\Omega(d)},
        \label{eq:Initialization-1}
    \end{align}
    where we used Lemma~\ref{lemma:chang2011chernoff} for the final inequality. 
    By letting $\beta=2$, we have that
 $
        \mathbb{P}[\kappa^0_j \geq \CC d^{-\frac12}] \gtrsim e^{-16\CC^2}.
    $
    Because of the symmetry, $\mathbb{P}[\kappa^0_j \leq  2\CC d^{-\frac12}]=\mathbb{P}[\kappa^0_j \geq 2\CC d^{-\frac12}]$.
\end{proofof}

\subsection{Weak Recovery: Population Update}\label{subsection:Expected}
We divide the first layer training into the first phase (weak recovery) and the second phase (strong recovery). 
We first evaluate the expected update of two gradient steps with the same training example.

\begin{lemma}\label{lemma:PopulationPoly}
    Let $\eta^{2t},\eta^{2t+1}= \eta = c_\eta d^{-1}$, $\xi^{2(t+1)}=\xi = 1-c_{\xi}d^{-\frac{(p_*-2)_+}{2}}$.
    Suppose that the link function satisfies $\mathrm{IE}(\sigma_*^I)=\mathrm{GE}(\sigma_*)=p_*$ (we choose the smallest such $I$) and activation functions satisfy all of the assumptions in Section~\ref{subsection:linkfunction} for weak recovery.  
    Then, for $\vw^{2t}$ with $c_\eta^{-1}d^{-\frac12}\leq \vtheta^\top \vw^{2t}\leq c_\eta^I$, the alignment $\vtheta^\top \vw^{2(t+1)}$ can be evaluated as,
    \begin{align}
      \vtheta^\top \vw^{2(t+1)} \geq \vtheta^\top \vw^{2t} + c_\eta^I  c_{\xi}c_\sigma d^{-\frac{p_*}{2}\lor 1}(\kappa^{2t})^{p_*-1}
        + c_\eta  c_{\xi} d^{-\frac{p_*}{2}\lor 1}\nu^{2t}.
    \end{align}
    Here $c_\sigma = p_*!\alpha_{p_*}\beta_{p_*}$ (when $\mathrm{IE}(\sigma_*)=\mathrm{GE}(\sigma_*)$) or $c_\sigma =\frac{p_*!H(\sigma_*^I;p_*)
    H(\sigma^{(I)}(\sigma^{(1)})^{I-1};p_*-1)}{2(I-1)!}$ (otherwise), and 
    $\nu^{2t}$ is a mean-zero sub-exponential random variable.
\end{lemma}

\begin{proof}
The expected alignment $\vtheta^\top \vw^{2(t+1)}$ after two gradient steps from $\vw^{2t}=\vomega$ using the same sample $(\vx,y)$, step size $\eta^{2t}=\eta^{2t+1}=\eta = c_\eta d^{-1}$ and momentum parameter $\xi^{2(t+1)}=\xi=1-c_{\xi}d^{-\frac{(p_*-2)_+}{2}}$ is evaluated as follows.
With a projection matrix $\vP_{\vomega}=\vI-\vomega\vomega^\top$, 
    the first step updates the weight as
    \begin{align}\label{eq:ExpectedUpdate-1-1}
        \vw^{2t+1} \leftarrow \vw^{2t}+\eta \tilde{\nabla}_{\vw} y\sigma({\vw^{2t}}^\top \vx)
        =\vomega+\eta y\sigma'(\vomega^\top \vx)\vP_{\vomega}\vx
        ,
    \end{align}
    and the next gradient step with the same sample is computed as
    \begin{align}
      \tilde{\nabla}_{\vw} y\sigma({\vw^{2t+1}}^\top \vx)
        &= y\sigma'({\vw^{2t+1}}^\top \vx)\vx
        \\ & = y\sigma'\big((\vomega+ \eta y\sigma'(\vomega^\top \vx)\vP_{\vomega}\vx)^\top \vx\big)\vP_{\vomega}\vx
         \\ & =   
         y\sigma'\big(\vomega^\top \vx + \eta\|\vx\|_{\vP_{\vomega}}^2\sigma'(\vomega^\top \vx)y\big)\vP_{\vomega}\vx,
         \label{eq:ExpectedUpdate-1-2}
    \end{align}
    here we used the notation $\|\vtheta\|_{A}^2=\vtheta^\top A \vtheta$ for a vector $\vtheta\in \R^d$ and a positive symmetric matric $A\in \R^{d\times d}$.
    From \eqref{eq:ExpectedUpdate-1-1} and \eqref{eq:ExpectedUpdate-1-2}, the parameter after the two steps is obtained as
    \begin{align}
        \vw^{2(t+1)} &\leftarrow \vw^{2t+1}+ \eta \tilde{\nabla}_{\vw} y\sigma({\vw^{2t+1}}^\top \vx)
        \\ & = \vomega +\eta y\sigma'(\vomega^\top \vx)\vP_{\vomega} \vx
         + \eta y\sigma'\big(\vomega^\top \vx + \eta \|\vx\|_{\vP_{\vomega}}^2
         \sigma'(\vomega^\top \vx)y\big)\vP_{\vomega}\vx.
         \\ & = \vomega +\eta \vg^{2t},
        \label{eq:ExpectedUpdate-1-3}
    \end{align}
    where
    \begin{align}
    \vg^{2t} = y\sigma'(\vomega^\top \vx)\vP_{\vomega} \vx
         + y\sigma'\big(\vomega^\top \vx + \eta\|\vx\|_{\vP_{\vomega}}^2
         \sigma'(\vomega^\top \vx)y\big)\vP_{\vomega}\vx
        \label{eq:ExpectedUpdate-1-16}.
    \end{align}
    
    Finally, the normalization step yields
    \begin{align}
        \vw^{2(t+1)} &\leftarrow \frac{\vw^{2(t+1)} - \xi^{2(t+1)} (\vw^{2(t+1)}-\vw^{2t}) }{\|\vw^{2(t+1)} - \xi^{2(t+1)}  (\vw^{2(t+1)}-\vw^{2t})\|}
        = \frac{\vomega + \eta\xi \vg^{2t} }{\|\vomega +  \eta\xi \vg^{2t}\|}
        = \frac{\vomega + c_\eta c_{\xi}d^{-\frac{p_*}{2}\lor 1}\vg^{2t} }{\|\vomega +  c_\eta c_{\xi}d^{-\frac{p_*}{2}\lor 1}\vg^{2t}\|}.
         \label{eq:ExpectedUpdate-1-14}
    \end{align}
    Therefore, by writing $\vtheta^\top \vw^{2t}=\kappa^{2t}$, the update of the alignment is 
    \begin{align}
     & \hspace{-7mm} \kappa^{2(t+1)}=\vtheta^\top \vw^{2(t+1)} 
     \\ &\hspace{-7mm} = \frac{\kappa^{2t} +  c_\eta c_{\xi}d^{-\frac{p_*}{2}\lor 1} \vtheta^\top \vg^{2t} }{\|\vomega + c_\eta  c_{\xi}d^{-\frac{p_*}{2}\lor 1}\vg^{2t}\|}
      \\ & \hspace{-7mm} \geq \kappa^{2t} + c_\eta  c_{\xi} d^{-\frac{p_*}{2}\lor 1}\vtheta^\top \vg^{2t}
       -\frac{1}{2}\kappa^{2t} c_\eta^2 c_{\xi}^2d^{-p_*\lor 2}\|\vg^{2t}\|^2
       -\frac{1}{2} c_\eta^3 c_{\xi}^3 d^{-\frac{3p_*}{2}\lor 3}|\vtheta^\top \vg^{2t}|\|\vg^{2t}\|^2
       .
         \label{eq:ExpectedUpdate-1-15}
    \end{align}
    We can easily see that $\mathbb{E}[\|\vg^{2t}\|^2]\lesssim d$ and $\mathbb{E}[|\vtheta^\top \vg^{2t}|\|\vg^{2t}\|^2]\lesssim d$, which implies that the expectation of the last two terms of \eqref{eq:ExpectedUpdate-1-15} is bounded by
    $\lesssim \kappa^{2t} c_\eta^2 c_{\xi}^2d^{-(p_*-1)\lor 1} \lor c_\eta^3 c_{\xi}^3 d^{-(\frac{3p_*}{2}-1)\lor 2}
    \leq  c_\eta^2 c_{\xi}^2d^{-(p_*-1)\lor 1}\kappa^{2t}$. 

Now we bound $\mathbb{E}[\vtheta^\top \vg^{2t}]$ by $\gtrsim c_\eta^{I-1}\kappa^{p_*-1}$.
Let $C_\sigma$ be the maximum degree of the activation function with non-zero coefficients of Hermite expansion, which may be infinity when we consider general link functions, and there appear some infinite sums. 
For these cases we simply assume the sums converge -- we discuss the validity of this condition in Section~\ref{subsubsection:Assumption-GeneralLink}. 
We omit the subscript $2t$ in the following for simplicity.
We divide the analysis into the two cases. 

\paragraph{(I) If $I=1 \Leftrightarrow \mathrm{IE}(\sigma_*)=\mathrm{GE}(\sigma_*)=p_*$.} 
For the first term of $\mathbb{E}[\vtheta^\top \vg]$, we have
    \begin{align}
        \vtheta^\top \mathbb{E}[ y\sigma'(\vomega^\top \vx)\vP_{\vomega} \vx]
&=\vtheta^\top \vP_{\vomega}\mathbb{E}\bigg[\bigg(\sum_{j=p_*}^\infty \alpha_j \He_j(\vtheta^\top \vx)\bigg)\bigg(\sum_{j=1}^{C_\sigma} j\beta_j \He_{j-1}\big(\vomega^\top \vx\big)\bigg)\vx\bigg]
       \\ & = \vtheta^\top \vP_{\vomega}\sum_{j=p_*}^\infty\bigg[ j! \alpha_j \beta_j \big(\vtheta^\top \vomega\big)^{j-1}\vtheta+ (j+2)! \alpha_j \beta_{j+2} \big(\vtheta^\top \vomega\big)^{j}\vomega\bigg]
       \\ & = \sum_{j=p_*}^{C_\sigma} j! \alpha_j \beta_j \big(\vtheta^\top \vomega\big)^{j-1}\vtheta^\top \vP_{\vomega}\vtheta
     \\ & = p_*!\alpha_{p_*}\beta_{p_*}\kappa^{p_*-1} + O(\kappa^{p_*}).
     \label{lemma:TaylorExpansion-1003}
    \end{align}
   For the second term of $\mathbb{E}[\vtheta^\top \vg]$, the following decomposition can be made.
    \begin{align}
      & \vtheta^\top\mathbb{E}[y\sigma'\big(\vomega^\top \vx + \eta\|\vx\|_{\vP_{\vomega}}^2\sigma'(\vomega^\top \vx)y\big)\vP_{\vomega}\vx]
      \\ & =\sum_{i=1}^{C_\sigma-1}(i!)^{-1}\vtheta^\top \vP_{\vomega}\mathbb{E}\bigg[y\sigma^{(i+1)}\big(\vomega^\top \vx\big)\big(\eta \|\vx\|_{\vP_{\vomega}}^2y\sigma^{(1)}(\vomega^\top \vx)\big)^{i}\vx\bigg]
        +  \vtheta^\top \mathbb{E}[ y\sigma'(\vomega^\top \vx)\vP_{\vomega} \vx]
      \\ & = 
      \sum_{i=1}^{C_\sigma-1}(i!)^{-1}\eta^{i}\vtheta^\top \vP_{\vomega}\mathbb{E}\bigg[\|\vx\|_{\vP_{\vomega}}^{2i}y^{i+1}\sigma^{(i+1)}\big(\vomega^\top \vx\big)\big(\sigma^{(1)}(\vomega^\top \vx)\big)^{i}\vx\bigg] \label{lemma:TaylorExpansion-43}
      \\
      &\quad + p_*!\alpha_{p_*}\beta_{p_*}\kappa^{p_*-1} + O(\kappa^{p_*})
      .\label{lemma:TaylorExpansion-100044}
    \end{align}
 We evaluate each term in the summation. 
We need to show that although $\|\vx\|_{\vP_{\vomega}}^2$ is a function of $\vtheta^\top \vx$ and $\vomega^\top\vx$, it is mostly independent from the two quantities.
To verify this, let $\ve =\frac{\vtheta-(\vtheta^\top \vomega)\vomega}{\|\vtheta-(\vtheta^\top \vomega)\vomega\|}$ be the orthogonal component of $\vtheta$ to $\vomega$.
Then, we have that
   \begin{align}
       \|\vx\|_{\vP_{\vomega}}^2 &= \vx^\top (I-\vomega\vomega^\top - \ve\ve^\top) \vx
       +(\ve^\top \vx)^2
       \\ & =\underbrace{\vx^\top (I-\vomega\vomega^\top - \ve\ve^\top) \vx}_{\sim \chi_{d-2}^2,\text{ independent from $\vomega^\top \vx$ and $\vtheta^\top \vx$}}
       +\bigg(\frac{\vtheta^\top \vx-(\vtheta^\top \vomega)\vomega^\top \vx}{\|\vtheta-(\vtheta^\top \vomega)\vomega\|}\bigg)^2.
       \label{eq:Orthogonal}
   \end{align}
   With $P_{\vomega,\vtheta}=\vI_d-\vomega\vomega^\top -\ve\ve^\top$, \eqref{lemma:TaylorExpansion-43} is expanded as
   \begin{align}
    \eqref{lemma:TaylorExpansion-43}
    &= \sum_{i=1}^{C_\sigma-1}\sum_{j=0}^{i}\sum_{l=0}^{i+1}\frac{{i \choose j}{i+1 \choose l}\eta^{i}}{i!\|\vtheta-(\vtheta^\top \vomega)\vomega\|^{2j}}   
    \\&\quad\quad\vtheta^\top \vP_{\vomega}\mathbb{E}\bigg[
       (\vx^\top P_{\vomega,\vtheta}\vx)^{i-j}\varsigma^{i+1-l}(\sigma_*(\vtheta^\top \vx))^l
     (\vtheta^\top \vx-(\vtheta^\top \vomega)\vomega^\top \vx)^{2j}\sigma^{(i+1)}\big(\vomega^\top \vx\big)\big(\sigma^{(1)}(\vomega^\top \vx)\big)^{i}\vx\bigg]
   \\  &  =\sum_{i=1}^{C_\sigma-1}\sum_{j=0}^{i}\sum_{l=0}^{i+1}\sum_{k=0}^{2j}\frac{{i \choose j}{i+1 \choose l}{2j\choose k}\eta^{i}\kappa^k(-1)^k\mathbb{E}[\varsigma^{i+1-l}]\mathbb{E}_{z\sim \chi^2_{d-2}}[z^{i-j}]}{i!\|\vtheta-(\vtheta^\top \vomega)\vomega\|^{2j}}
        \\&\quad\quad\mathbb{E}\bigg[(\sigma_*(\vtheta^\top \vx))^l
        (\vtheta^\top \vx)^{2j-k}
       (\vomega^\top \vx)^k\sigma^{(i+1)}\big(\vomega^\top \vx\big)\big(\sigma^{(1)}(\vomega^\top \vx)\big)^{i}(\vtheta^\top\vx-\vtheta^\top\vomega \vomega^\top \vx)\bigg]
       \label{lemma:TaylorExpansion-2004}
   \end{align}

   For a general differentiable function $g(\vx)$, we have
   $\mathbb{E}[\He_t (x_1)g(\vx)]=\mathbb{E}[\frac{\mathrm{d}^t}{\mathrm{d}x_1^t}g(\vx)]$.
   If $g(\vx)$ is a polynomial (with a bounded coefficients) of $x_1$ and $\vu^\top \vx$ and its degree with respect to $x_1$ is at most $s (\leq t)$, $|\mathbb{E}[\He_t (x_1)g(\vx)]| \lesssim |u_1|^{t-s}$, because differentiation of $g(\vx)=\bar{g}(x_1,\vu^\top \vx)$ is taken with respect to the first variable at most $s$ times.
   Each term of \eqref{lemma:TaylorExpansion-2004} is an expectation of $(\sigma_*(\vtheta^\top \vx))^l$, multiplied by the polynomial of $\vtheta^\top \vx$ and $\vomega^\top \vx$, where its degree with respect to $\vtheta^\top \vx$ is at most $2j-k$.
   Thus each term of \eqref{lemma:TaylorExpansion-2004} is evaluated as (here we omit the constants)
   \begin{align}
    &  \frac{\eta^i\kappa^k(-1)^k\mathbb{E}_{z\sim \chi^2_{d-2}}[z^{i-j}]}{i!\|\vtheta-(\vtheta^\top \vomega)\vomega\|^{2j}} \mathbb{E}\bigg[ \underbrace{ (\sigma_*(\vtheta^\top \vx))^l}_{\mathrm{IE}\geq p_*}
        \underbrace{ (\vtheta^\top \vx)^{2j-k}
       (\vomega^\top \vx)^k\sigma^{(i+1)}\big(\vomega^\top \vx\big)\big(\sigma^{(1)}(\vomega^\top \vx)\big)^{i}(\vtheta^\top\vx-\vtheta^\top\vomega \vomega^\top \vx)}_{\text{degree w.r.t. $\vtheta^\top \vx$ is at most $2j-k+1$}}\bigg]
      \\ & \lesssim c_\eta^i d^{-i}d^{i-j}\kappa^k \kappa^{((p_*-2j+k-1)\lor 0)}
      \lesssim c_\eta d^{-j}\kappa^{((p_*-2j-1)\lor 0)}
      \leq c_\eta \kappa^{p_*-1} (d/\kappa^2)^{-j} \leq c_\eta \kappa^{p_*-1}. 
   \end{align}
   The lower bound follows in the same fashion.
   Therefore, 
   \begin{align}
       |\eqref{lemma:TaylorExpansion-2004}| \lesssim c_\eta \kappa^{p_*-1}.
   \end{align}
   Now, $ \mathbb{E}[\vtheta^\top\vg]$ can be evaluated as
   \begin{align}
 \mathbb{E}[\vtheta^\top\vg]= \eqref{lemma:TaylorExpansion-1003}+\eqref{lemma:TaylorExpansion-43}+\eqref{lemma:TaylorExpansion-100044}
       =2p_*!\alpha_{p_*}\beta_{p_*}\kappa^{p_*-1}   + O(c_\eta \kappa^{p_*-1}+\kappa^{p_*}).
   \end{align}
   
\paragraph{(II) If $I=\{\min\ i\mid \mathrm{IE}(\sigma_*^i)=\mathrm{GE}(\sigma_*)=p_*\}\geq 2$.} 
 Note that $\alpha_j=0$ for all $j\leq p_*$ from the assumption. 
 Following \eqref{lemma:TaylorExpansion-1003}, the first term of $\mathbb{E}[\vtheta^\top \vg]$ is evaluated as
    \begin{align}
        \vtheta^\top \mathbb{E}[ y\sigma'(\vomega^\top \vx)\vP_{\vomega} \vx]
        = \sum_{j=p}^{C_\sigma} j! \alpha_j \beta_j \big(\vtheta^\top \vomega\big)^{j-1}\vtheta^\top \vP_{\vomega}\vtheta =
      O(\kappa^{p_*}).
        \label{lemma:TaylorExpansion-21}
    \end{align}
  For the second term of  $\mathbb{E}[\vtheta^\top \vg]$, similarly to \eqref{lemma:TaylorExpansion-2004},
  the following decomposition can be made.
    \begin{align}
     & \vtheta^\top\mathbb{E}[y\sigma'\big(\vomega^\top \vx + \eta \|\vx\|_{\vP_{\vomega}}^2\sigma'(\vomega^\top \vx)y\big)\vP_{\vomega}\vx]
      \\ &=   \sum_{i=0}^{C_\sigma-1}\sum_{j=0}^{i}\sum_{l=0}^{i+1}\sum_{k=0}^{2j}\frac{{i \choose j}{i+1 \choose l}{2j\choose k}\eta^{i}\kappa^k(-1)^k\mathbb{E}[\varsigma^{i+1-l}]\mathbb{E}_{z\sim \chi^2_{d-2}}[z^{i-j}]}{i!\|\vtheta-(\vtheta^\top \vomega)\vomega\|^{2j}}
        \\&\quad\quad\mathbb{E}\bigg[(\sigma_*(\vtheta^\top \vx))^l
        (\vtheta^\top \vx)^{2j-k}
       (\vomega^\top \vx)^k\sigma^{(i+1)}\big(\vomega^\top \vx\big)\big(\sigma^{(1)}(\vomega^\top \vx)\big)^{i}(\vtheta^\top\vx-\vtheta^\top\vomega \vomega^\top \vx)\bigg].
             \label{lemma:TaylorExpansion-23}
    \end{align}
   Each term of \eqref{lemma:TaylorExpansion-23} (omitting constants) is evaluated as 
      \begin{align}
     &  \eta^i \kappa^k \mathbb{E}_{z\sim \chi^2_{d-2}}[z^{i-j}]\mathbb{E}\bigg[(\sigma_*(\vtheta^\top \vx))^l
        (\vtheta^\top \vx)^{2j-k}
       (\vomega^\top \vx)^k\sigma^{(i+1)}\big(\vomega^\top \vx\big)\big(\sigma^{(1)}(\vomega^\top \vx)\big)^{i}(\vtheta^\top\vx-\vtheta^\top\vomega \vomega^\top \vx)\bigg]
       \label{lemma:TaylorExpansion-26}
            \\ & =c_\eta^i \kappa^k d^{-j}
    \mathbb{E}\bigg[
      \underbrace{ (\sigma_*(\vtheta^\top \vx))^l}_{\mathrm{IE}\geq \begin{cases}
          p_*\ (l\geq I)\\  p_*+1\ (l< I)
      \end{cases}}
      \underbrace{ (\vtheta^\top \vx)^{2j-k} (\vomega^\top \vx)^k\sigma^{(i+1)}\big(\vomega^\top \vx\big)\big(\sigma^{(1)}(\vomega^\top \vx)\big)^{i}(\vtheta^\top\vx-\vtheta^\top\vomega \vomega^\top \vx)}_{\text{degree w.r.t. $\vtheta^\top \vx$ is at most $2j-k+1$}}\bigg]
      \\ & \lesssim c_\eta^i \kappa^k d^{-j}\kappa^{(\mathrm{IE}(\sigma_*^l)-2j+k-1)\lor 0}
     \label{lemma:TaylorExpansion-31}
   \end{align}
    When $i\leq I-2$ and $\eta= c_\eta d^{-1}$, we have $l\leq i+1< I$ and $\mathrm{IE}((\sigma_*(\vtheta^\top \vx))^l)\geq p_*+1$. Thus
   \begin{align}
       \eqref{lemma:TaylorExpansion-31} \lesssim 
      \kappa^{p_*}
       \label{lemma:TaylorExpansion-28}
   \end{align}
   When $i\geq I$, $\mathrm{IE}((\sigma_*(\vtheta^\top \vx))^l)\geq p_*$ and we get
   \begin{align}
    \eqref{lemma:TaylorExpansion-31}
      \lesssim c_\eta^I \kappa^{p_*-1}.
    \label{lemma:TaylorExpansion-29}
   \end{align} 
   Now the case of $i=I-1$. 
    When $i=I-1$ and $j \ne 0$, and using the assumption that $\kappa \leq c_\eta$, 
     \begin{align}
    \eqref{lemma:TaylorExpansion-31}
      \lesssim c_\eta^{I-1} \kappa^{p_*-1}(\kappa^{-2}/d) \leq c_\eta^{I} \kappa^{p_*-1}. 
    \label{lemma:TaylorExpansion-130}
   \end{align} 
     When $i=I-1$, $j = 0$, and $k\ne 0$, 
     \begin{align}
    \eqref{lemma:TaylorExpansion-31}
      \lesssim c_\eta^{I-1} \kappa^{p_*}.
    \label{lemma:TaylorExpansion-131}
   \end{align}   
When $i=I-1$, $j = 0$, $k= 0$, and $l\leq I-1$, 
     \begin{align}
    \eqref{lemma:TaylorExpansion-31}
      \lesssim c_\eta^{I-1} \kappa^{p_*}.
    \label{lemma:TaylorExpansion-132}
   \end{align}  
   Therefore, except for $i=I-1$, $j = 0$, $k= 0$, and $l\leq I-1$, we can bound \eqref{lemma:TaylorExpansion-31} by $\lesssim c_\eta^{I} \kappa^{p_*-1} + \kappa^{p_*}$. The lower bound follows in the same way. 
   Finally, consider the case of $i=I-1$, $j=0$, $k=0$, and $l=I$. 
   \begin{align}
       \eqref{lemma:TaylorExpansion-26}
      & =
      \eta^{I-1} \mathbb{E}_{z\sim \chi^2_{d-2}}[z^{I-1}]
      \mathbb{E}\bigg[(\sigma_*(\vtheta^\top \vx))^I\sigma^{(I+1)}\big(\vomega^\top \vx\big)\big(\sigma^{(1)}(\vomega^\top \vx)\big)^{I-1}(\vtheta^\top\vx-\vtheta^\top\vomega \vomega^\top \vx)\bigg]
     \\ & =  
     \eta^{I-1} \mathbb{E}_{z\sim \chi^2_{d-2}}[z^{I-1}]
    \sum_{m=p_*}^{C_\sigma I}
     m!  H(\sigma_*^I;m)
       H(\sigma^{(I)}(\sigma^{(1)})^{I-1};m-1)(1-\kappa^2)\kappa^{m-1}
        \\ & = c_\eta^{I-1} p_*! d^{-(I-1)}\mathbb{E}_{z\sim \chi^2_{d-2}}[z^{I-1}] H(\sigma_*^I;p_*)
       H(\sigma^{(I)}(\sigma^{(1)})^{I-1};p_*-1)(1-\kappa^2)\kappa^{p_*-1}  +O( c_\eta^{I-1} \kappa^{p_*}).
   \end{align}
    Putting it all together (recovering the constants omitted in \eqref{lemma:TaylorExpansion-26} again), 
    \begin{align}
      &  \eqref{lemma:TaylorExpansion-23}
      \\ &=c_\eta^{I-1}\frac{p_*! d^{-(I-1)}\mathbb{E}_{z\sim \chi^2_{d-2}}[z^{I-1}] }{(I-1)!}H(\sigma_*^I;p_*)
       H(\sigma^{(I)}(\sigma^{(1)})^{I-1};p_*-1)\kappa^{p_*-1} 
       +O(c_\eta^{I} \kappa^{p_*-1} +\kappa^{p_*}),
    \end{align}
    and 
     \begin{align}
      & \mathbb{E}[\vtheta^\top \vg] =  \eqref{lemma:TaylorExpansion-21} + \eqref{lemma:TaylorExpansion-23}
      \\ &=c_\eta^{I-1}\underbrace{\frac{p_*!d^{-(I-1)}\mathbb{E}_{z\sim \chi^2_{d-2}}[z^{I-1}] }{(I-1)!}}_{=\Theta(1)}H(\sigma_*^I;p_*)
       H(\sigma^{(I)}(\sigma^{(1)})^{I-1};p_*-1)\kappa^{p_*-1} 
       +O(c_\eta^{I} \kappa^{p_*-1} +\kappa^{p_*}).
    \end{align}

    Combining (i) and (ii), we have 
    \begin{align}
        \mathbb{E}[\vtheta^\top \vg] \geq 2c_\eta^{I-1} c_\sigma \kappa^{p_*-1}+O(c_\eta^{I} \kappa^{p_*-1} +\kappa^{p_*})
    \end{align}
    for a positive constant $c_\sigma = \Theta(1)$. 
    Here $c_\sigma>0$ satisfies $2c_\sigma = 2p_*!\alpha_{p_*}\beta_{p_*}$ (for (i)) or $2c_\sigma =\frac{p_*!H(\sigma_*^I;p_*)
       H(\sigma^{(I)}(\sigma^{(1)})^{I-1};p_*-1)}{(I-1)!}$ (for (ii)). 
    Going back to \eqref{eq:ExpectedUpdate-1-15}, by setting $\nu^{2t}=(\vtheta^\top \vg^{2t}-\mathbb{E}[\vtheta^\top \vg^{2t}])$, we have
    \begin{align}
    \kappa^{2(t+1)}
    &\geq 
        \kappa^{2t} + 2c_\eta  c_{\xi} d^{-\frac{p_*}{2}\lor 1}\mathbb{E}[\vtheta^\top \vg^{2t}]
        + c_\eta  c_{\xi} d^{-\frac{p_*}{2}\lor 1}(\vtheta^\top \vg^{2t}-\mathbb{E}[\vtheta^\top \vg^{2t}])
       +O(c_\eta^2 c_{\xi}^2d^{-(p_*-1)\lor 1}\kappa^{2t})
      \\ &=
      \kappa^{2t} + 2c_\eta^I c_{\xi} d^{-\frac{p_*}{2}\lor 1}c_\sigma(\kappa^{2t})^{p_*-1}
        + c_\eta  c_{\xi} d^{-\frac{p_*}{2}\lor 1}\nu^{2t}
       \\ & \quad\quad  +O\left(c_\eta^2 c_{\xi}^2d^{-(p_*-1)\lor 1}(\kappa^{2t})^{2t}+c_\eta^{I+1} c_\xi d^{-\frac{p_*}{2}\lor 1}(\kappa^{2t})^{p_*-1}+c_\eta  c_\xi d^{-\frac{p_*}{2}\lor 1}(\kappa^{2t})^{p_*}\right)
       .
    \end{align}
    When $c_\xi \leq c_\eta^I$ and $c_\eta^{-1}d^{-\frac12} \leq \kappa\leq c_\eta^I$, terms in the big-$O$ notation is smaller than $c_\eta^I c_{\xi} d^{-\frac{p_*}{2}\lor 1}c_\sigma(\kappa^{2t})^{p_*-1}$ and we have
    \begin{align}
        \kappa^{2(t+1)} \geq \kappa^{2t} + c_\eta^I  c_{\xi} c_\sigma d^{-\frac{p_*}{2}\lor 1}(\kappa^{2t})^{p_*-1}
        + c_\eta  c_{\xi} d^{-\frac{p_*}{2}\lor 1}\nu^{2t}.
    \end{align}
    It is straightforward to check $\nu^{2t}$ has sub-Weibull tail.
    \end{proof}
    
\subsection{Weak Recovery: Stochastic Update}\label{subsection:Stochastic}
This subsection proves weak recovery using the results on population update from the previous section. 
Specifically, from the previous section, we know that
\begin{align}
      \vtheta^\top \vw^{2(t+1)} \geq \vtheta^\top \vw^{2t} + c_\eta^I  c_{\xi}c_\sigma d^{-\frac{p_*}{2}\lor 1}(\kappa^{2t})^{p_*-1}
        + c_\eta  c_{\xi} d^{-\frac{p_*}{2}\lor 1}\nu^{2t},
\end{align}
with the mean-zero sub-Weibull random variable $\nu^{2t}$ and a positive $c_\sigma = \Theta(1)$.
For notational simplicity we write $c_\eta^I c_\sigma = c_1$. The following lemma is a detailed version of Proposition~\ref{prop:StochasticGE}. 
\begin{lemma}\label{lemma:StochasticGE}
        Take $\eta^{2t},\eta^{2t+1}= \eta = c_\eta d^{-1}$, $\xi^{2(t+1)}=\xi = 1-c_{\xi}d^{-\frac{(p_*-2)_+}{2}}$.
        Suppose that the link function satisfies $\mathrm{IE}(\sigma_*^I)=\mathrm{GE}(\sigma_*)=p_*$ (we choose the smallest such $I$)  and activation functions satisfy all of the assumptions in Section~\ref{subsection:linkfunction} for weak recovery.
        Let
        \begin{align}
            T_{1,1} = C_3c_\xi^{-1} \begin{cases}
                d & (\text{if $p_*=\mathrm{GE}(\sigma_*)=1$})
                \\ 
                d (\log d) &(\text{else if $p_*=\mathrm{GE}(\sigma_*)=2$})
                \\
                d^{p_*-1} &(\text{else $p_*=\mathrm{GE}(\sigma_*)\geq 3$}),
            \end{cases}
        \end{align}
        and take $c_\xi \lesssim \delta \mathrm{poly}(c_\eta)$, $c_2 \gtrsim \mathrm{poly}(c_\eta)$, and $C_3\simeq c_1^{-1}$. 
        If $\kappa^{0}\geq 2c_\eta^{-1}d^{-\frac12}$, there exists some $\tau_*\leq T_{1,1}$ such that
        \begin{align}
            \kappa^{2\tau_*} \geq 2c_2 ,
        \end{align}
        with probability at least $1-\delta$, 
        and $\kappa^{2\tau}\geq 2c_2$ for all $\tau_*\leq \tau\leq T_{1,1}$, with high probability.
\end{lemma}
We may take $\delta=o_d(1)$ with arbitrarily slow decay. 
The proof is adapted from \cite{arous2021online}, but our bound on $T_{1,1}$ is slightly sharper (by a $\log d$ factor for $p_*=2$ and by a $(\log d)^2$ factor for $p_*\geq 3$).
For $p_*=2$, this is because of a trick that we carefully ``restart'' the dynamics, whose failure probability exponentially decays.

    \begin{proof} We divide the proof into the following cases. 
    
    \noindent {\bf (i) When $p_*=1$.}
        Note that $\{\sum_{s=0}^{\tau} \nu^{2s}\}_{\tau}$ is  Martingale with $\mathbb{E}[(\nu^{2s})^2]\lesssim 1$.
        By Doob’s maximal inequality and Markov's inequality, with probability $1-\delta$, 
        we have
        \begin{align}
           \max_{0\leq \tau\leq T} \bigg|\sum_{s=0}^{\tau} \nu^{2s}\bigg|^2 \leq \delta^{-1}\mathbb{E}[(\sum_{s=0}^{T} \nu^{2s})^2]\leq \delta^{-1}\sum_{s=0}^{T}\mathbb{E}[(\nu^{2s})^2]\leq C_1 \delta^{-1}(T+1)
           \label{eq:MarkovDeBound}
        \end{align}
        for any fixed $T\geq 0$, with a sufficiently large constant $C_1=\Theta(1)$.
        In the following we consider the case when \eqref{eq:MarkovDeBound} holds for $T=c_1^{-1}c_\xi^{-1}d-1$. 

    If $c_\eta^{-1}d^{-\frac12}\leq \kappa^{2t} \leq c_\eta^I$ for all $t=0,1,\cdots,\tau$, we have
        \begin{align}
             \kappa^{2(\tau+1)} &\geq \kappa^{2\tau} +   c_1 c_{\xi}d^{-1}  +  c_\eta  c_{\xi}d^{-1}  \nu^{2\tau}
             \label{eq:AP-Stochastic-6}
             \\ & \geq 2c_\eta^{-1} d^{-\frac12} +c_1 c_{\xi}(\tau+1) d^{-1} \gamma -   c_\eta  c_{\xi} d^{-1} \bigg|\sum_{s=0}^{\tau} \nu^{2s}\bigg|
             .\label{eq:AP-Stochastic-3}
        \end{align}
        Now, applying \eqref{eq:MarkovDeBound} to get
        \begin{align}
           \kappa^{2(\tau+1)}\geq   \eqref{eq:AP-Stochastic-3}
            \geq 2 c_\eta^{-1} d^{-\frac12} +c_1c_{\xi}(\tau+1) d^{-1}  -  c_\eta c_{\xi}^\frac12 c_1^{-\frac12} C_1^{\frac12}\delta^{-\frac12} d^{-\frac12},
        \end{align}
       when $\tau \leq c_1^{-1}c_\xi^{-1}d-1$. 
        By letting $c_{\xi}\leq c_\eta^{-4}c_1C_1^{-1} \delta$, 
        we have $c_\eta^{-1} d^{-\frac12} \leq c_\eta c_{\xi}^\frac12 c_1^{-\frac12} C_1^{\frac12}\delta^{-\frac12} d^{-\frac12}$, and 
        \begin{align}
            \kappa^{2(\tau+1)}\geq c_\eta^{-1} d^{-\frac12} +c_1c_{\xi}(\tau+1) d^{-1}  ,
        \end{align}
        which verifies $c_\eta^{-1}d^{-\frac12}\leq \kappa^{2t}$ for $t = \tau+1$.
        Thus, there exists some $\tau_* \leq c_1^{-1}c_\xi^{-1}d$ such that
        \begin{align}
            \kappa^{2\tau_*}\geq 4c_2,
        \end{align}
        for $c_1 \leq \frac14 c_\eta^I$, with probability $1-\delta$.

        Now we prove that $\kappa^{2t}\geq 2c_2$ holds for all $\tau_* \leq t\leq T_{1,1}=C_3c_\xi^{-1}  d $.
        Because $\nu^{2t}$ are mean-zero sub-Weibull random variables, we also have that $|\sum_{s=\tau}^{\tau+\tau'-1} \nu^{2s}| \leq C_4 \sqrt{\tau'}$ for all $0\leq \tau,\tau'\leq T_{1,1}$ with high probability.
        Also, because $\eta^{t}\ll d^{-1}$ and $|1-\xi^{t}| \ll 1$, we can easily see that $|\kappa^{2(\tau+1)}-\kappa^{2\tau}| =\tilde{O}( d^{-1})$ for all $\tau = 0,1,\cdots,T_{1,1}-1$, with high probability.
        Thus, when there exists $\tau\geq \tau_*$ such that $\kappa^{2(\tau-1)}\geq 4c_2$ and $\kappa^{2\tau}< 4c_2$, we have $\kappa^{2\tau}\geq 3c_2$ with high probability.
        Moreover, following the above argument, we can inductively show that
        \begin{align}
            \kappa^{2(\tau+\tau')}&\geq 3c_2 +c_1c_{\xi}\tau' d^{-1} -  c_\eta c_\xi d^{-1} C_4 \sqrt{\tau'}
            \\ &\geq 
            3c_2 +c_1c_{\xi}\tau' d^{-1}  -
            \begin{cases}
                c_2 & (\tau' \leq c_\eta^{-2}c_\xi^{-2}C_4^{-2}c_2^2d^2)
\\ 
                 c_1c_{\xi}\tau' d^{-1}  & (\tau' \geq c_\eta^2 c_1^{-2} C_4^2)
            \end{cases}.
            \\ & \geq 2c_2,
            \label{eq:AfterCare}
        \end{align}
        for $\tau' \leq T_{1,1}=C_3 c_\xi^{-1} d $ or until $\kappa^{2(\tau+\tau')}\geq 4c_2$ holds again. 
        By repeating this argument (if there are multiple such $\tau$), we see that $\kappa^{2t}\geq 2c_2$ holds for all $\tau_* \leq t\leq T_{1,1}=C_3c_\xi^{-1}  d $ with high probability.

    \noindent {\bf (ii) When $p_*=2$.}
        We define $\iota_0 = 0, \iota_1 = \log_{(1+c_1 c_{\xi} d^{-1})}(4), \iota_2 = 2\log_{(1+c_1 c_{\xi} d^{-1})}(4), \dots$.
        We show that, for each $i$, if $\kappa^{2\iota_i}\geq 2c_\eta^{-1}d^{-\frac12}$, we have $\kappa^{2(\iota_{i+1})}\geq 2 \kappa^{2\iota_i}$, with probability at least $1-\delta 4^{-i}$, or there exists some $t \ (\iota_i< t \leq \iota_{i+1})$ with $\kappa^{2t}>c_\eta^I$. 

        Assume that the above statement holds until some $i-1\geq 0$ (we do not need to assume anything for $i=0$).
        Then, we have $\kappa^{2\iota_i}\geq 2^i \kappa^{0}\geq 2c_\eta^{-1}d^{-\frac12}$.
        Similarly to \eqref{eq:AP-Stochastic-3}, if $c_\eta^{-1}d^{-\frac12}\leq \kappa^{2t} \leq c_\eta^I$ for all $t=\iota_i,\iota_i+1,\cdots,\tau$, we have
    \begin{align}
             \kappa^{2(\tau+1)} 
            \geq \kappa^{2\iota_i} +c_1 c_{\xi}d^{-1} \sum_{s=\iota_i}^\tau \kappa^{2s} - c_\eta c_{\xi} d^{-1} \bigg|\sum_{s=\iota_i}^{\tau} \nu^{2s}\bigg|
             .
        \end{align}
    Applying \eqref{eq:MarkovDeBound} with $\delta = \delta/4^i$ and $T=\frac14 c_\eta^{-2}c_\xi^{-2}C^{-1}(\delta/4^i) (\kappa^{2\iota_i})^2 d^{2}-1$ to get
        \begin{align}
           \kappa^{2(\tau+1)}
           &
            \geq 
            \kappa^{2\iota_i} +c_1 c_{\xi}d^{-1} \sum_{s=\iota_i}^\tau \kappa^{2s} - c_\eta c_\xi d^{-1}C^\frac12 \delta^{-\frac12}\sqrt{\tau+1-\iota_i}
           \\
           & \geq \kappa^{2\iota_i} +c_1 c_{\xi}d^{-1} \sum_{s=\iota_i}^\tau \kappa^{2s}  - \frac12\kappa^{2\iota_i}
        \end{align}
        when $\tau \leq \iota_i + \frac14 c_\eta^{-2}c_\xi^{-2}C^{-1}(\delta/4^i) (\kappa^{2\iota_i})^2 d^{2}-1$, 
        which verifies $c_\eta^{-1}d^{-\frac12}\leq \frac12 \kappa^{2\iota_i} \leq \kappa^{2t}$ for $t = \tau+1$.
        
        This implies that, with probability $1-\delta/ 4^{i}$, we have
        \begin{align}
            \kappa^{2(\tau+1)}
            \geq \frac12 \kappa^{2\iota_i} +c_1 c_{\xi}d^{-1} \sum_{s=\iota_i}^\tau \kappa^{2s}
        \end{align}
        for all $\tau = \frac14 c_\eta^{-2}c_\xi^{-2}C^{-1}(\delta/4^i) (\kappa^{2\iota_i})^2 d^{2}-1$, which is equivalent to
        \begin{align}
             \kappa^{2\tau} \geq (1+c_1 c_{\xi} d^{-1})^{\tau-\iota_i} \frac12 \kappa^{2\iota_i}
        \end{align}
        for all $\tau = \iota_i,\iota_i+1,\cdots,\iota_i+\frac14 c_\eta^{-2}c_\xi^{-2}C^{-1}(\delta/4^i) (\kappa^{2\iota_i})^2 d^{2}$.
        By taking $c_\xi \ll c_1 c_\eta^{-2}C^{-1}(\delta/4^i) (\kappa^{2\iota_i})^2 d$, we have $\frac14 c_\eta^{-2}c_\xi^{-2}C^{-1}(\delta/4^i) (\kappa^{2\iota_i})^2 d^{2}\geq \log_{(1+c_1 c_{\xi} d^{-1})}(4)$, and we get
        \begin{align}
            \kappa^{2\iota_{i+1}} \geq 2 \kappa^{2\iota_{i}}
        \end{align}
        with probability $1-\delta/ 4^{i}$ (or there exists $t\leq \iota_{i+1}$ such that $\kappa^{2t}>c_\eta^I$).

        Thus, by induction, for all $i$, we have that
        \begin{align}
            \kappa^{2\iota_{i}} \geq 2^i \kappa^0,
            \label{eq:Exponentioal-Growth}
        \end{align}
        or that there exists some $t\leq \iota_{i}$ such that $\kappa^{2t}$ is larger than $ c_\eta^I$, with probability $1-\delta$.
        
        The LHS of \eqref{eq:Exponentioal-Growth} becomes larger than $ c_\eta^I$ for some $i \leq \log d$.
        Because $\iota_i = \Theta (ic_1^{-1}c_\xi^{-1} d)$, within $O(c_1^{-1}c_\xi d\log d)$ steps, there exists at least one $\tau_* = O(c_1^{-1}c_\xi^{-1} d\log d)$ such that $ \kappa^{2\tau_*} \geq 4c_2$
        for $c_2 \leq \frac14 c_\eta^I$, with probability $1-\delta$.

        Once such $\tau_*$ is obtained, following the last paragraph of (i), we can see that $\kappa^{2t}\geq 2c_2$ holds until $t = T_{1,1}$ with high probability.

\noindent {\bf (iii) When $p_*\geq 3$.} 
    We apply \eqref{eq:MarkovDeBound} with $T = \frac{1}{p_*-2}c_1^{-1} c_{\xi}^{-1} d^{\frac{p_*}{2}} (\kappa^{0})^{-(p_*-2)}$ to obtain that
    \begin{align}
      c_\eta c_\xi d^{-\frac{p_*}{2}} \bigg|\sum_{s=0}^{\tau} \nu^{2s}\bigg| \leq  c_\eta c_{\xi}^{\frac12} c_1^{-\frac12}C^\frac12 \delta^{-\frac12}d^{-\frac{p_*}{4}} (\kappa^{0})^{-\frac{p_*-2}{2}}
      \label{eq:NoiseLevel}
    \end{align}
    for all $\tau =0,1,\cdots,T-1$, with probability $1-\delta$.

    We take $c_{\xi} \ll c_\eta^{-2}c_1C^{-1}\delta d^{\frac{p_*}{4}} (\kappa^{0})^{\frac{p_*}{2}}$ so that \eqref{eq:NoiseLevel} is bounded by $\cA^{-1}d^{-\frac12}$. 
    Then, 
    \begin{align}
   \kappa^{2(\tau+1)}
  & \geq \kappa^{0}+ c_1 c_{\xi} d^{-\frac{p_*}{2}}\sum_{s=0}^\tau (\kappa^{2s})^{p_*-1}
        + c_\eta  c_{\xi} d^{-\frac{p_*}{2}}\sum_{s=0}^\tau \nu^{2s}
   \\ & \geq c_\eta^{-1}d^{-\frac12} +  c_1 c_{\xi} d^{-\frac{p_*}{2}}\sum_{s=0}^\tau (\kappa^{2\tau})^{p_*-1}.          
       \end{align}
       It is easy to see that $\kappa^{2(\tau+1)}$ is lower bounded by $a^{\tau+1}$, where $a^0 = c_\eta^{-1}d^{-\frac12}$ and $a^{\tau+1}=a^\tau + c_1 c_{\xi} d^{-\frac{p_*}{2}} (a^{\tau})^{p_*-1}$.
       By applying Lemma~\ref{lemma:Bihari–LaSalle}, we have
       \begin{align}
           \kappa^{2\tau} \geq \frac{\kappa^0}{\left(1- c_1 c_{\xi} d^{-\frac{p_*}{2}}(p_*-2)(\kappa^0)^{(p_*-2)}t\right)^\frac{1}{p_*-2}}.
       \end{align}
    Thus, until $\tau \leq \big(c_1 c_{\xi} d^{-\frac{p_*}{2}} (p_*-2)(\kappa^0)^{(p_*-2)}\big)^{-1}\leq T+1 \ll d^{p_*-1}$, with probability at least $1-\delta$, there exists some $\tau_*$ such that
    \begin{align}
        \kappa^{2\tau_*} \geq 4c_2 \geq c_\eta^I
    \end{align}
    when $c_2 \leq \frac14 c_\eta^I$.

     Once such $\tau_*$ is obtained, following the last paragraph of (i), we can see that $\kappa^{2t}\geq 2c_2$ holds until $t = T_{1,1}$ with high probability.
    \end{proof}

In the above proof we used the (discrete version of) Bihari–LaSalle inequality from \cite{ben2022high}.
\begin{lemma}\label{lemma:Bihari–LaSalle}
    For $p\geq 3$ and $c>0$, consider a positive sequence $(a^t)_{t\geq 0}$ such that
     \begin{align}
         a^{t+1} = a^t +c (a^t)^{p-1}.
     \end{align}
     Then, we have
     \begin{align}\label{eq:BihariLaSalle}
         a^t \geq \frac{a^0}{\left(1-c (p-2) (a^0)^{(p-2)}t\right)^\frac{1}{p-2}}.
     \end{align}
\end{lemma}
\begin{proof}
     From definition, we have
        \begin{align}
             c = \frac{a^{t+1}-a^t}{(a^t)^{p-1}}
            \leq \int_{t=a^{t}}^{a^{t+1}}\frac{1}{x^{p-1}}\leq \frac{1}{p-2}\left[\frac{1}{(a^t)^{p-2}}- \frac{1}{(a^{t+1})^{p-2}}\right].
        \end{align}
        Taking the summation and re-arranging the terms yield
        \begin{align}
            (a^t)^{-(p-2)} \leq (a^0)^{-(p-2)} -c (p-2) t,
            \\ \therefore a^t \geq \frac{a^0}{\left(1- c (p-2) (a^0)^{(p-2)}t\right)^\frac{1}{p-2}},
        \end{align}
        which gives the lower bound.
\end{proof}

\subsection{From Weak Recovery to Strong Recovery}\label{subsection:Final}
In the previous subsection, we proved that after $t=2T_{1,1}=\tilde{\Theta}(d)$ steps, with probability $\tilde{\Omega}(1)$ over the randomness of initialization, we obtain nontrivial alignment
$\kappa^{2T_{1,1}}_j \geq 2c_2$.
This subsection discusses how to convert the weak recovery into the strong recovery.
\begin{lemma}
    Suppose the neuron satisfies $\kappa^{2T_{1,1}} \geq 2c_2$.  
    Take $\eta^{2t}= \eta = \bar{c}_\eta \varepsilon d^{-1}$, $\eta^{2t+1}=0$, $\xi^{2(t+1)}=0$ for all $t\geq T_{1,1}$, where $\bar{c}_\eta \lesssim \mathrm{poly}(c_1)$. 
    If the activation functions satisfy all of the assumptions in Section~\ref{subsection:linkfunction} for strong recovery, then we have
    \begin{align}
        \vtheta^\top \vw^{2(T_{1,1}+\tau_*)} \geq 1-\varepsilon,
    \end{align}
    with high probability, where $\tau_*\leq T_{1,2} = C_3 d \varepsilon^{-2}$.
    Moreover, $\vtheta^\top \vw^{2(T_{1,1}+t)} \geq 1-\varepsilon$ for all $\tau_*\leq t \leq T_{1,2}= C_3 d \varepsilon^{-2}$, with high probability.
\end{lemma}
\begin{proof}
    Consider the Hermite expansions of $\sigma_*$ and $\sigma$.
    Let $p$ be the smallest degree that both $\sigma_*$ and $\sigma$ have non-zero coefficients.
    First we compute the population gradient (of the correlation term) as
    \begin{align}
        \mathbb{E}\big[\tilde{\nabla}_{\vw} y\sigma({\vw^{2t}}^\top \vx) \big]
      &  =  \mathbb{E}\bigg[\tilde{\nabla}_{\vw} \bigg(\sum_{j=p}^{\infty}\alpha_j \He_j(\vtheta^\top \vx)\bigg)\bigg( \sum_{j=0}^{\infty}\beta_j \He_j({\vw^{2t}}^\top \vx)\bigg)\bigg]
      \\ &  =  \sum_{j=p}^{\infty} \big[j! \alpha_j \beta_j(\vtheta^\top \vw^{2t})^{j-1}\vtheta + (j+2)! \alpha_j \beta_{j+2}(\vtheta^\top \vw^{2t})^j\vw^{2t}\big].
        \label{eq:StrongRecovery-2}
    \end{align}
    Applying $P_{\vw^{2t}}$, we have
    \begin{align}
    \mathbb{E}\big[P_{\vw^{2t}}\tilde{\nabla}_{\vw} y\sigma({\vw^{2t}}^\top \vx) \big]
    =
      (\vtheta-({\vw^{2t}}^\top \vtheta)\vw^{2t}) \sum_{j=p}^{\infty} j! \alpha_j \beta_j(\vtheta^\top \vw^{2t})^{j-1}.
      \label{eq:StrongRecovery-1}
    \end{align}
    Thus, the update of the alignment $\kappa^{2t}=\vtheta^\top \vw^{2t}$ is
    \begin{align}
     \kappa^{2(t+1)} \geq \kappa^{2t} + \eta \vtheta^\top \vg
       -\frac{1}{2}\eta^2\kappa^{2t} \|\vg\|^2
       -\frac{1}{2}\eta^3 \tilde{\eta}^3 |\vtheta^\top \vg|\|\vg\|^2
     ,
    \end{align}
    where
    \begin{align}
    \vg = P_{\vw^{2t}}y\sigma'({\vw^{2t}}^\top \vx) \vx. 
    \end{align}
    From \eqref{eq:StrongRecovery-2}, the expectation of \eqref{eq:StrongRecovery-1} is bounded by
    \begin{align}
    \mathbb{E} [\kappa^{2(t+1)}] &\geq  \kappa^{2t} + \eta 
      (1-(\kappa^{2t})^2) 
      \sum_{j=p}^{\infty} j! \alpha_j \beta_j(\vtheta^\top \vw^{2t})^{j-1}
      - \eta^2C_4 d(\kappa^{2t}+\eta)
      \\ & \geq \kappa^{2t} + \eta(1-(\kappa^{2t})^2)  \sum_{j=p}^\infty j!\alpha_j \beta_j (\kappa^{2t})^{p-1} -  \eta^2 C_4 d(\kappa^{2t}+\eta).
    \end{align}
    By letting $\eta\leq c_1^{p-1}\varepsilon d^{-1}$, when $\kappa^{2t}\leq 1-\varepsilon$, we have
    \begin{align}
        \mathbb{E} [\kappa^{2(t+1)}]  \geq  \kappa^{2t} + \frac12\eta  \varepsilon \sum_{j=p}^\infty j!\alpha_j \beta_j (\kappa^{2t})^{p-1} \geq \kappa^{2t}+\eta  \varepsilon c_1^{p}.
    \end{align}
    It is easy to see that the noise $\nu^{2t}$ has sub-Weibull tail and we obtain that
    \begin{align}
      \kappa^{2(t+1)}  \geq \kappa^{2t} + \frac12\eta  \varepsilon \sum_{j=p}^\infty j!\alpha_j \beta_j (\kappa^{2t})^{p-1} + \eta \nu^{2t}\geq\kappa^{2t}+ \eta  \varepsilon c_1^{p}+ \eta \nu^{2t}.
      \label{eq:OneStepStrongRecovery}
    \end{align}
    Suppose that $2c_2\leq \kappa^{2(T{1,1}+\tau)}\leq 1-\varepsilon$ for all $t=0,1,\dots,\tau-1$.
    By taking the summation of \eqref{eq:OneStepStrongRecovery}, we have
    \begin{align}
        \kappa^{2(T{1,1}+\tau)}
        \geq \kappa^{2T{1,1}}+ \eta  \varepsilon tc_1^{p}+ \eta \sum_{s=T{1,1}}^{T{1,1}+\tau-1}\nu^{2t}
        \geq 2c_2+ \eta  \varepsilon \tau c_1^{p}- C_4 \eta \sqrt{\tau}
        , \label{eq:OneStepStrongRecovery-2}
    \end{align}
    with high probability.
    The third term is bounded by $C_4 \eta \sqrt{\tau}\leq c_2$ when $\tau \leq c_2^2 C_4^{-2} \eta^{-2} = c_2^2 C_4^{-2}\bar{c}_\eta^{-2} \varepsilon^{-2} d^2$ and by  $\frac12\eta  \varepsilon \tau c_1^{p}$ when $\tau \geq 4\varepsilon^{-2} c_1^{-2p} C_4^2$.
    Because $c_2^2 C_4^{-2}\bar{c}_\eta^{-2} \varepsilon^{-2} d^2\geq 4\varepsilon^{-2} c_1^{-2p} C_4^2$, we can bound \eqref{eq:OneStepStrongRecovery-2} by
    \begin{align}
      \kappa^{2(T{1,1}+\tau)} \geq c_2+ \frac12\eta  \varepsilon \tau c_1^{p},
      \label{eq:OneStepStrongRecovery-3}
    \end{align}
    which verifies $2c_2\leq \kappa^{2(T{1,1}+\tau)}$.
    
    Therefore, by induction, until $\kappa^{2t}\geq 1-\varepsilon$, we have the lower bound \eqref{eq:OneStepStrongRecovery-3}, whose RHS exceeds $1-\varepsilon$ when $\tau \geq 2\eta^{-1}\varepsilon^{-1}c_1^{-p}\leq C_3 d \varepsilon^{-2}$.
    Thus, there exists $\tau_* \leq T_{1,2} = C_3 d \varepsilon^{-2}$ such that $\kappa^{2(T{1,1}+\tau_*)}\geq 1-\varepsilon$, with high probability.

    Now, what remains is to prove that $\kappa^{2(T{1,1}+\tau)}\geq 1-3\varepsilon$ holds for all $\tau_* \leq t\leq T_{1,2}=C_3 d \varepsilon^{-2}$.
        Because $\nu^{2t}$ are mean-zero sub-Weibull random variables, we have that $|\sum_{s=\tau}^{\tau+\tau'-1} \nu^{2s}| \leq C_4 \sqrt{\tau'}$ for all $0\leq \tau,\tau'\leq T_{1,1}$ with high probability.
        Also, because $\eta^{t}\ll \varepsilon d^{-1}$, we can easily see that $|\kappa^{2(\tau+1)}-\kappa^{2\tau}| =\tilde{O}( \varepsilon d^{-1})$ for all $\tau = 0,1,\cdots,T_{1,1}-1$, with high probability.
        Thus, when there exists $\tau\geq \tau_*$ such that $\kappa^{2(T_{1,1}+\tau-1)}\geq 1-\varepsilon$ and $\kappa^{2(T_{1,1}+\tau)}< 1-\varepsilon$, we have $\kappa^{2(T_{1,1}+\tau)}\geq 1-2\varepsilon$ with high probability.
        Moreover, following the above argument, we can inductively show that
        \begin{align}
            \kappa^{2(T{1,1}+\tau+\tau')}&\geq 1-2\varepsilon +\eta  \varepsilon \tau' c_1^{p}- C_4 \eta \sqrt{\tau'}
            \\ &\geq 
            1-2\varepsilon +\eta  \varepsilon \tau' c_1^{p}
           - \begin{cases}
                \varepsilon & (\tau' \leq \bar{c}_\eta^{-2}C_4^{-2}d^2)
\\ 
               \eta  \varepsilon \tau' c_1^{p} & (\tau' \geq \varepsilon^{-2}C_4^2c_1^{-2p})
            \end{cases}.
            \\ & \geq 1-3\varepsilon,
            \label{eq:AfterCare-2}
        \end{align}
        for $\tau' \leq T_{1,2}$ or until $\kappa^{2(T{1,1}+\tau+\tau')}\geq 1-\varepsilon$ holds again.
        Note that the last inequality follows from $\bar{c}_\eta^{-2}C_4^{-2}d^2 \geq  \varepsilon^{-2}C_4^2c_1^{-2p}$. 
        By repeating this argument (if there are multiple such $\tau$), we can see that $\kappa^{2(T{1,1}+t)}\geq 1-\varepsilon$ holds for all $\tau_* \leq t\leq T_{1,2}=C_3 d \varepsilon^{-2} $ with high probability.

        Adjusting hidden constants to remove a factor of $3$ from $3\varepsilon$ yields the desired result.
\end{proof}

\subsection{Second Layer Training}\label{subsection:Second}
From the previous analysis, we know that at least $\Omega(1)$ portion of the neurons will satisfy the weak and strong recovery conditions (Appendix~\ref{subsection:linkfunction}), at least $\tilde{\Omega}(1)$ portion of the neurons (independent from the choice of $\sigma_j$) satisfy initial alignment conditions (Appendix~\ref{subsection:Initialization}), and at least $1-o(1)$ fraction of them achieves strong recovery. 
This subsection proves a generalization error bound after second-layer training. 
Let $f_{\va}(\vx) = f_{\vTheta}(\vx)$ for $\vTheta=(\hat{\vw}_j,a_j,\hat{b}_j)_{j=1}^N$ where 
$\va \in \R^N$ and $(\hat{\vw}_j,\hat{b}_j)_{j=1}^N$ are the parameters trained in the first stage. 
Let $\va^* \in \R^N$ be the ``certificate'' with $\|\va^*\|^2 = \tilde{O}(N)$ that is shown to exist in Lemma~\ref{lemma:DamianApproximationPolynomial}. 

\paragraph{Polynomial Link Functions.} The following lemma is a complete version of Proposition~\ref{prop:2nd-layer-training}.

\begin{lemma}
\label{lemma:Generalization}

There exists a $4q$-th order polynomial $Q(R_{\vw},b,q')$ of $R_{\vw} = \max_j\norm{\vw_j}$, $b=(b_j)_{j=1}^N$ such that,  
if we set $\lambda = \Theta\left(\sqrt{\frac{2}{T_2 \delta_0} N^2 Q(R_{w},b,q')}\right)$ for some $\delta_0 > 0$, the ridge estimator $\hat{\va}$ satisfies 
\begin{align}
    \|f_{\hat{\va}} - f_*\|_{L^2(P_x)}^2 
    &\lesssim 
     (N^{-2}+\varepsilon^2) 
    + \frac{1}{T_2 \lambda \delta_0} \left( 2 N^2 Q(R_{w},b,q')  + \E_{\vx}[(f_*)^4]\right) + 
    \frac{3\lambda}{2}  \|\va^*\|^2,
    \label{eq:GeneralizationError-1} 
\end{align}
with probability $1 - \delta_0$. 
Hence taking 
$T_2=\tilde{\Theta}((N^4 Q^2(R_{w},b,q')  + \E[f_*(\vx)^4]^2) \varepsilon^{-4})$ 
and $N=\tilde{\Theta}( \varepsilon^{-1})$, 
we have $$\mathbb{E}_{\vx}[(f_{\hat{\va}}(\vx) - f_*(\vx))^2]\lesssim \varepsilon^2.$$
\end{lemma}

\begin{proof} 
Let $P_{T_2}$ be the empirical distribution of the second stage: $P_{T_2} := \frac{1}{T_2}\sum_{i=1}^{T_2} \delta_{\vx_i}$.
Let $\psi(\vx) = (\sigma(\langle \vx,\widehat{\vw}_j)\rangle + b_j))_{j=1}^N$ so that 
$f_{\va}(\vx) = \langle\va,\psi(\vx)\rangle$. 

\paragraph{Part (1).} 
We first bound the term $\|f_{\hat{\va}} - f_*\|_{L^2(P_{T_2})}$. 
Since 
 $\hat{\mathcal{L}}(f_{\hat{\va}})  + \lambda \|\hat{\va}\|^2 
\leq \hat{\mathcal{L}}(f_{\va^*}) + \lambda \|\va^*\|^2$, we have
\begin{align}\label{eq:fhatfstarFirstEval}
&\|f_{\hat{\va}} - f_*\|_{L^2(P_{T_2})}^2
 + \lambda \|\hat{\va}\|^2 \\
&\leq \|f_{\va^*} - f_*\|_{L^2(P_{T_2})}^2 + \frac{2}{T_2}\sum_{i=1}^{T_2}(f_{\va^*}(\vx_i) - f_{\hat{\va}}(\vx_i))\varepsilon_i
 + \lambda \|\va^*\|^2.
\end{align}
Now, by the Cauchy-Schwarz inequality, we have 
\begin{align*}
\frac{2}{T_2}\sum_{i=1}^{T_2}(f_{\va^*}(x_i) - f_{\hat{\va}}(x_i))\varepsilon_i 
& =
(\va^* - \hat{\va})^\top \frac{2}{T_2}\sum_{i=1}^{T_2} \psi(\vx_i) \varepsilon_i \\
& \leq 
2 \|\va^* - \hat{\va}\| \sqrt{\frac{\sum_{i,j}\varepsilon_i \varepsilon_j \psi(\vx_i)^\top \psi(\vx_j)}{T_2^2}}.
\end{align*}
By applying Markov's inequality to the right hand side, it can be further bounded by 
$$
\|\va^* - \hat{\va}\| \sqrt{\frac{\mathbb{E}_{\vx}[\|\psi(\vx)\|^2]}{T_2 \delta_1}}
\leq \frac{\lambda}{2}\|\hat{\va}\|^2 + \frac{\lambda}{2} \|\va^*\|^2
+ \frac{\mathbb{E}_{\vx}[\|\psi(\vx)\|^2]}{T_2 \delta_1 \lambda},
$$
with probability $1-\delta_1$.  
Thus, by combining with \eqref{eq:fhatfstarFirstEval}, we arrive at 
$$
\|f_{\hat{\va}} - f_*\|_{L^2(P_{T_2})}^2
 + \frac{\lambda}{2} \|\hat{\va}\|^2 
\leq \|f_{\va^*} - f_*\|_{L^2(P_{T_2})}^2 + \frac{\mathbb{E}_{\vx}[\|\psi(\vx)\|^2]}{T_2 \delta_1 \lambda}
 + \frac{3\lambda }{2} \|\va^*\|^2.
$$
Here, by using the evaluation $\|f_{\va^*} - f_*\|_{L^2(P_{T_2})}=\tilde{O}(N^{-1}+\varepsilon)$ in Lemma~\ref{lemma:DamianApproximationPolynomial}, the right hand side can be further bounded by 
$$
\|f_{\hat{\va}} - f_*\|_{L^2(P_{T_2})}^2
 + \frac{\lambda}{2} \|\hat{\va}\|^2 
\leq 
\tilde{O}(N^{-2}+\varepsilon^2) + \frac{\mathbb{E}_{\vx}[\|\psi(\vx)\|^2]}{T_2 \delta_1 \lambda}
 + \frac{3\lambda }{2} \|\va^*\|^2.
$$

\paragraph{Part (2).} Next we lower bound $\|f_{\hat{\va}} - f_*\|_{L^2(P_{T_2})}^2$ by noticing that
\begin{align}
 & \|f_{\hat{\va}} - f_*\|_{L^2(P_{T_2})}^2  \\
 & = 
 \|f_{\hat{\va}} - f_*\|_{L^2(P_{T_2})}^2 - 
\|f_{\hat{\va}} - f_*\|_{L^2(P_x)}^2 
+ 
\|f_{\hat{\va}} - f_*\|_{L^2(P_x)}^2 \\
& = 
 \|f_{\hat{\va}} \|_{L^2(P_{T_2})}^2 - \|f_{\hat{\va}} \|_{L^2(P_x)}^2 
- 2 \left(\frac{1}{T_2} \sum_{i=1}^{T_2} f_{\hat{\va}}(\vx_i)f_*(\vx_i) - \mathbb{E}[f_{\hat{\va}}(\vx_i)f_*(\vx_i)]\right) \\
& ~~~~~+ \|f_*\|_{L^2(P_{T_2})}^2 -\|f_*\|_{L^2(P_x)}^2 
+ 
\|f_{\hat{\va}} - f_*\|_{L^2(P_x)}^2.
\label{eq:LoneLossFirstBound2}
\end{align}
The first two terms of Eq. \eqref{eq:LoneLossFirstBound2} can be bounded by 
\begin{align*}
\left|  \|f_{\hat{\va}} \|_{L^2(P_{T_2})}^2 - \|f_{\hat{\va}} \|_{L^2(P_x)}^2 \right| 
& = \left|\hat{\va}^\top \left(\frac{\sum_{i=1}^{T_2} \psi(\vx_i)\psi(\vx_i)^{\top}}{T_2}
 - \mathbb{E}_{\vx}[\psi(\vx)\psi(\vx)^{\top}]
 \right) \hat{\va} \right| \\
& \leq \|\hat{\va}\|^2 \sup_{\va : \|\va\| \leq 1} \left|\|f_{\va} \|_{L^2(P_{T_2})}^2 - \|f_{\va} \|_{L^2(P_x)}^2\right|.
\end{align*}
The standard Rademacher complexity bound yields that  
\begin{align}
& \E_{(x_i)_{i=1}^{T_2}}\left[ \sup_{\va \in \R^N:\|\va\| \leq 1}
\left| \|f_{\va} \|_{L^2(P_x)}^2 
-\|f_{\va} \|_{L^2(P_{T_2})}^2  \right| \right]  \\
\leq &
2  \E_{(x_i,\sigma_t)_{t=1}^{T_2}}\left[ \sup_{\va \in \R^N:\|\va\| \leq 1}
\left| \frac{1}{T_2} \sum_{t=1}^{T_2} \sigma_t f_{\va}(x_i)^2  \right| \right] \\ 
\leq &
2 \sqrt{ \E_{(\vx_i)_{i=1}^{T_2}}\left[ \sup_{\va \in \R^N:\|\va\| \leq 1}
 \frac{1}{T_2^2} \sum_{i=1}^{T_2} (\va^\top \psi(\vx_i) )^4 \right]  } \\ 
\leq &
2 \sqrt{ \E_{(\vx_i)_{i=1}^{T_2}}\left[ 
 \frac{1}{T_2^2} \sum_{i=1}^{T_2} \|\psi(\vx_i)\|^4 \right]  } \\
= & 
2 \sqrt{
 \frac{1}{T_2}   \E_{\vx}[\|\psi(\vx)\|^4]  }, 
 \end{align}
where $(\sigma_i)_{i=1}^{T_2}$ is the i.i.d. Rademacher sequence independent of $(\vx_i)_{i=1}^{T_2}$.  
Hence, Markov's inequality yields 
\begin{align*}
\left|  \|f_{\hat{\va}} \|_{L^2(P_{T_2})}^2 - \|f_{\hat{\va}} \|_{L^2(P_x)}^2 \right| 
& = 2 \|\hat{\va}\|^2 \sqrt{
 \frac{1}{T_2 \delta_2}   \E_{\vx}[\|\psi(\vx)\|^4]  }, 
\end{align*}
with probability $1- \delta_2$. 

The third term in Eq. \eqref{eq:LoneLossFirstBound2} can be evaluated as  
\begin{align*}
&2 \left(\frac{1}{T_2} \sum_{i=1}^{T_2} f_{\hat{\va}}(\vx_i)f_*(\vx_i) - \mathbb{E}_{\vx}[f_{\hat{\va}}(\vx)f_*(\vx)]\right)  \\
&= \hat{\va}^\top \left( \frac{1}{T_2} \sum_{i=1}^{T_2} (\psi(\vx_i)f_*(\vx_i) - \mathbb{E}_{\vx}[\psi(\vx)f_*(\vx)] ) \right)   \\
& \leq \|\hat{\va}\| 
\sqrt{\frac{1}{T_2^2} \sum_{i=1}^{T_2}\sum_{j=1}^{T_2} 
(\psi(\vx_i)f_*(\vx_i) - \mathbb{E}_{\vx}[\psi(\vx)f_*(\vx)])^\top
(\psi(\vx_j)f_*(\vx_j) - \mathbb{E}_{\vx}[\psi(\vx)f_*(\vx)])
} \\
& \leq 
\|\hat{\va}\| \sqrt{\frac{1}{T_2 \delta_3}  
\mathbb{E}_{\vx}[\|\psi(\vx)f_*(\vx) - \mathbb{E}_{\vx}[\psi(\vx)f_*(\vx)]\|^2]} \\
& \leq 
\|\hat{\va}\| \sqrt{\frac{1}{T_2 \delta_3}  
\mathbb{E}_{\vx}[\|\psi(\vx)\|^4 + \|f_*(\vx) \|^4]}  \\
& \leq 
\frac{\lambda}{4} \|\hat{\va}\|^2 
+ \frac{1}{\lambda T_2 \delta_3}  
\mathbb{E}_{\vx}[\|\psi(\vx)\|^4 + \|f_*(\vx) \|^4],
\end{align*}
with probability $1- \delta_3$ where we used Markov's inequality again in the second inequality. 

Finally, the fourth and fifth term in Eq. \eqref{eq:LoneLossFirstBound2} can be bounded as 
\begin{align*}
\left|\|f_*\|_{L^2(P_{T_2})}^2 -\|f_*\|_{L^2(P_x)}^2\right| 
& = 
\sqrt{\left(\|f_*\|_{L^2(P_{T_2})}^2 -\|f_*\|_{L^2(P_x)}^2\right)^2} \\
& \leq 
\sqrt{\frac{1}{T_2 \delta_4} \mathbb{E}_{\vx}
[(f^{*}(\vx)^4 -\|f_*\|_{L^2(P_x)}^2)^2]} \\
& \leq 
\sqrt{\frac{1}{T_2 \delta_4} \mathbb{E}_{\vx}
[(f^{*}(\vx))^4]},
\end{align*}
with probability $1-\delta_4$ where we used Markov's inequality in the last inequality. 

Combining these inequalities, we finally arrive at 
\begin{align*}
& \|f_{\hat{\va}} - f_*\|_{L^2(P_x)}^2
 + \left( \frac{\lambda}{4}  
 - \sqrt{\frac{2}{T_2 \delta_2}   \E_{\vx}[\|\psi(\vx)\|^4]  }
 \right) \|\hat{\va}\|^2  \\
& \leq 
\tilde{O}(N^{-2}+\varepsilon^2) + \frac{1}{T_2  \lambda}
\left( \frac{\mathbb{E}_{\vx}[\|\psi(\vx)\|^2]}{\delta_1} 
+  \frac{\mathbb{E}_{\vx}[\|\psi(\vx)\|^2]}{\delta_3} 
+  \frac{\mathbb{E}_{\vx}[(f^*(\vx))^4]}{\delta_3} 
\right)
 + \frac{3\lambda }{2} \|\va^*\|^2,
\end{align*}
with probability $1- \sum_{j=1}^4\delta_j$. 
Hence, by setting $\lambda \geq 8 \sqrt{\frac{2}{T_2 \delta_2}   \E_{\vx}[\|\psi(\vx)\|^4]}$,
we have that 
\begin{align*}
& \|f_{\hat{\va}} - f_*\|_{L^2(P_x)}^2 \\
& \leq 
\tilde{O}(N^{-2}+\varepsilon^2) + \frac{1}{T_2  \lambda}
\left( \frac{\mathbb{E}_{\vx}[\|\psi(\vx)\|^2]}{\delta_1} 
+  \frac{\mathbb{E}_{\vx}[\|\psi(\vx)\|^4]}{\delta_3} 
+  \frac{\mathbb{E}_{\vx}[(f^*(\vx))^4]}{\delta_3} 
\right)
 + \frac{3\lambda }{2} \|\va^*\|^2.
\end{align*}
When the activation function $\sigma$ is a polynomial, then each $\psi_j(\vx) = \sigma(\langle \vx,\vw_j \rangle + b_j)$ is an order $q$-polynomial of a Gaussian random variable $\langle \vx,\vw_j \rangle$ which has mean 0 and variance $\E[\langle \vx,\vw_j \rangle^2] = \|\vw_j\|^2 =\tilde{O}(1)$. 
Then, if we let $R_{w} := \max_{j} \|\vw_j\| =\tilde{O}(1)$, 
the term $\max_j \max\{\E_{\vx}[\psi(\vx)_j^2],\E_{\vx}[\psi(\vx)_j^4]\}$ can be bounded by a $4q$-th order polynomial of $R_{w}$ and $b$, which can be denoted by $Q(R_{w},b,4q)$. 

\paragraph{Part (3).} By combining evaluations of (1) and (2) together, 
if we let $\lambda = 8 \sqrt{\frac{2}{T_2 \delta_0}   \E_{\vx}[\|\psi(\vx)\|^4]}$ for some $\delta_0 > 0$, 
(by ignoring polylogarithmic factors) we obtain that
\begin{align}
    \!\!\!\!\!\!\|f_{\hat{\va}} - f_*\|_{L^2(P_x)}^2 
    &\lesssim 
     (N^{-2}+\varepsilon^2) 
    + \frac{1}{T_2 \lambda \delta_0} \left( 2 N^2 Q(R_{w},b,q')  + \E_{\vx}[(f_*(\vx))^4]\right) + 
    \frac{3\lambda}{2}  \|\va^*\|^2,
    % \label{eq:GeneralizationError-1} 
\end{align}
with probability $1 - 4 \delta_0$. 
Thus, since $\|\va^*\|^2 = \tilde{O}(N)$, by setting $T_2=\tilde{\Theta}((N^4 Q^2(R_{w},b,q')  + \E[f_*(\vx)^4]^2) \varepsilon^{-4})$, 
and $N=\tilde{\Theta}(\varepsilon^{-1})$, we obtain that $\eqref{eq:GeneralizationError-1}\lesssim \varepsilon^2$.
\end{proof}

\paragraph{Higher Generative Exponent Functions.}
For general link functions, under Assumption~\ref{assump:Approx-General} and the bounded fourth moment of the link function, we have the following counterpart of Lemma~\ref{lemma:Generalization}, which provides the formal statement of Proposition~\ref{prop:high-GE-test-error}.
\begin{lemma}
Suppose that $\mathbb{E}[\sigma_*(\vtheta^\top x)^4]<\infty$ and Assumption~\ref{assump:Approx-General} hold.
Then, by setting  $\lambda = \tilde{\Theta}\left(\sqrt{\frac{N^2 }{T_2\delta_0} }\right)$ for some $\delta_0 > 0$, the ridge estimator $\hat{\va}$ satisfies 
\begin{align}
    \|f_{\hat{\va}} - f_*\|_{L^2(P_x)}^2 
    &\lesssim 
   \varepsilon^2 
    + \frac{1}{\sqrt{T_2\delta_0}} \left(  N^2 C_4  + \E_{\vx}[(f_*)^4]\right) + 
    \frac{1}{\sqrt{T_2\delta_0}}  \|\va^*\|^2,
    \label{eq:GeneralizationError-2} 
\end{align}
with probability $1 - \delta_0$. 
By taking 
$T_2=\tilde{\Theta}(( N^4 +N^2) \varepsilon^{-4})$, 
we have $$\mathbb{E}_{\vx}[(f_{\hat{\va}}(\vx) - f_*(\vx))^2]\lesssim \varepsilon^2.$$

Furthermore, applying Lemma~\ref{lemma:DesignActivation-2}~and~\ref{lemma:ReLU_settles_everything} yields that, when $\sigma_*=\sum_{j=0}^\infty \alpha_j \He_j$ satisfies $\sum_{j=0}^\infty j^2 j! \alpha_j^2$ and $\mathbb{E}[\sigma_*(\vtheta^\top x)^4]$ are bounded, with a properly designed randomized activation in Lemma~\ref{lemma:DesignActivation-2}, by taking $N = \tilde{\Theta}(\varepsilon^{-7})$ and $T_2=\tilde{\Theta}(\varepsilon^{-32})$, Algorithm~1 yields $$\mathbb{E}_{\vx}[(f_{\hat{\va}}(\vx) - f_*(\vx))^2]\lesssim \varepsilon^2,$$ with probability $1-o_d(1)$. 
\end{lemma}
\begin{proof}
    The proof is identical to that of Lemma~\ref{lemma:Generalization}, with the difference being that we replace the bounded moment assumptions with $\mathbb{E}[\sigma_*(\vtheta^\top \vx)^4]<\infty$ or Assumption~\ref{assump:Approx-General}.
\end{proof}

\paragraph{Approximation Guarantee.}
Note that for non-polynomial link function with generative exponent $p_*\ge 3$, the approximation error is already controlled in Assumption~\ref{assump:Approx-General} based on \cite[Lemma 4.4, 4.5]{bietti2022learning} (using activation function with a ReLU component). 
If $\sigma_*$ is a degree-$q$ polynomial, we have the following approximation result using polynomial activation, which follows Lemmas 29 and 30 of \cite{oko2024learning}. 
\begin{lemma}\label{lemma:DamianApproximationPolynomial}
    Suppose that there exist at least $N'=\tilde{\Theta}(N)$ neurons that satisfy $\|\vw_j^{2T_1}-\vtheta\|\leq \varepsilon$ and $\sigma$ is a polynomial link function with degree at least $q$.
    Let $b_j \sim \mathrm{Unif}([-C_b,C_b])$ with $C_b=\tilde{O}(1)$ , and consider approximation of a ridge function $h(\vtheta^\top \vx)$ with its degree at most $q$.
    Then, there exists $a_1,\dots,a_{N}$ such that
    \begin{align}
       \left| \frac{1}{N}\sum_{j=1}^N a_j \sigma_j \big({\vw_j^{2T_1}}^\top \vx +b_j\big)
       - h(\vtheta^\top \vx)\right|=\tilde{O}(N^{-1}+\varepsilon)
    \end{align}
    with high probability, where $(\vx,y)$ is a random sample, and we omit dependence on the degree $q$ in the big-$O$ notation. 
    Moreover, we have $\sum_{j=1}^{N} a_j^2 = \tilde{O}(N)$.
\end{lemma}
Lemma~\ref{lemma:DamianApproximationPolynomial} can be established from the following result in \cite{oko2024learning}.
\begin{lemma}\label{lemma:ApproxPolynomialByHeq}
    Suppose that $C_b\geq q$.
    For any polynomial $h(s)$ with its degree at most $q$, there exists $\bar{v}(b;h)$ with $|\bar{v}(b;h)| \lesssim C_b$ such that for all $s$,
    \begin{align}
        \mathbb{E}[\bar{v}(b;h)\sigma(\delta s+b)]=h(s).
    \end{align}
\end{lemma}

\end{document}